\documentclass{article}

\usepackage{microtype}
\usepackage{graphicx}
\usepackage{subfigure}
\usepackage{booktabs} %
\usepackage{amsfonts}       %
\usepackage{nicefrac}       %
\usepackage{microtype}      %
\usepackage{xcolor}         %

\usepackage[para]{threeparttable}
\usepackage{multirow}
\usepackage{multicol}

\usepackage{graphicx}
\usepackage{adjustbox}
\usepackage{amsmath}
\usepackage{mathtools}
\usepackage{bm}
\usepackage{xfrac}
\usepackage{pifont}

\usepackage{hyperref}

\newcommand{\xmark}{\ding{55}}%

\usepackage[accepted]{icml2024_arxiv}

\usepackage{amsmath}
\usepackage{amssymb}
\usepackage{mathtools}
\usepackage{amsthm}

\usepackage[capitalize,noabbrev]{cleveref}
\makeatletter
\newcommand\Label[1]{&\refstepcounter{equation}(\theequation)\cref@label{#1}&}
\makeatother

\usepackage[textsize=tiny]{todonotes}

\Crefname{equation}{Eq.}{Eqs.}
\Crefname{figure}{Fig.}{Figs.}
\Crefname{tabular}{Tab.}{Tabs.}
\Crefname{appendix}{App.}{Apps.}

\usepackage{amsmath,amsfonts,bm,amsthm}
\usepackage{mathtools}

\newtheorem{proposition}{Proposition}

\newcommand\Bh{\bm{h}}

\newcommand\Bm{\bm{m}}

\newcommand\Bv{\bm{v}}
\newcommand\Bw{\bm{w}}

\newcommand\Bz{\bm{z}}

\newcommand\BA{\bm{A}}

\newcommand\BH{\bm{H}}

\newcommand\BP{\bm{P}}

\newcommand\BR{\bm{R}}

\newcommand\BV{\bm{V}}

\newcommand\Rx{\mathbf{x}}
\newcommand\Ry{\mathbf{y}}
\newcommand\Rz{\mathbf{z}}
\newcommand\Rt{\mathbf{t}}

\newcommand\RX{\mathbf{X}}

\newcommand\RZ{\mathbf{Z}}
\newcommand\RR{\mathbf{R}}

\newcommand\Bphi{\bm{\phi}}

\newcommand\Fcal{\mathcal{F}}

\icmltitlerunning{VN-EGNN: E(3)-Equivariant Graph Neural Networks with Virtual Nodes Enhance Protein Binding Site Identification}

\begin{document}

\twocolumn[
\icmltitle{VN-EGNN: E(3)-Equivariant Graph Neural Networks with Virtual Nodes Enhance Protein Binding Site Identification}

\icmlsetsymbol{equal}{*}

\begin{icmlauthorlist}
\icmlauthor{Florian Sestak}{iml}
\icmlauthor{Lisa Schneckenreiter}{iml}
\icmlauthor{Johannes Brandstetter}{iml,nxai}
\icmlauthor{Sepp Hochreiter}{iml,nxai}
\icmlauthor{Andreas Mayr}{iml}
\icmlauthor{Günter Klambauer}{iml}
\end{icmlauthorlist}

\icmlaffiliation{iml}{ELLIS Unit Linz, Institute for Machine Learning, JKU Linz, Austria}
\icmlaffiliation{nxai}{NXAI GmbH, Linz, Austria}

\icmlcorrespondingauthor{Florian Sestak}{sestak@ml.jku.at}

\icmlkeywords{Machine Learning, ICML}

\vskip 0.3in
]

\printAffiliationsAndNotice{}  %

\begin{abstract}
Being able to identify regions within or around proteins,
to which ligands can potentially bind, 
is an essential step to develop new drugs. 
Binding site identification methods can now profit from 
the availability of large amounts of 3D structures in protein structure 
databases or from AlphaFold predictions.
Current binding site identification methods heavily rely on 
graph neural networks (GNNs), 
usually designed to output E($3$)-%
equivariant predictions.
Such methods turned out to be very beneficial for physics-related 
tasks like binding energy or motion trajectory prediction. 
However, the performance of GNNs at binding site identification 
is still limited potentially
due to the lack of dedicated nodes 
that model hidden geometric entities, 
such as binding pockets. 
In this work, we extend E($n$)-Equivariant Graph Neural Networks (EGNNs) 
by adding virtual nodes and applying an extended message passing scheme. 
The virtual nodes in these graphs are dedicated quantities 
to learn representations of binding sites, 
which leads to improved predictive performance. 
In our experiments, 
we show that our proposed method 
VN-EGNN sets a new 
state-of-the-art 
at locating binding site centers on 
COACH420, HOLO4K and PDBbind2020.
\end{abstract}

\section{Introduction} %
\label{Introduction}

\textbf{Binding site identification remains a central computational problem in drug discovery.}
With the advent of AlphaFold \citep{Jumper2021}, millions of 3D structures of proteins have been unlocked for further 
investigation by the scientific community \citep{Tunyasuvunakool2021, Cheng2023}. 
Information about the 3D structure of a protein 
can provide crucial information about its function. One of the most important fields
that should profit from these 3D structures, is drug discovery \citep{Ren2023, Sadybekov2023}. 
It has been envisioned that
the availability of 3D structures will allow to purposefully design drugs that alter
protein function in a desired way. However, to enable structure-based drug design, further
computational approaches have to be utilized/employed, concretely either \emph{docking} or 
\emph{binding site identification} methods \citep{Lengauer1996, Cheng2007, Halgren2009}. 
While docking approaches predict the location of a specific
small molecule, called ligand, within a protein’s active site upon binding, binding
site identification aims at finding regions on the protein likely to form a
binding pocket and interact with unknown ligands \citep{Schmidtke2010}.
Note that docking and binding site identification are fundamentally 
different tasks in structure-based drug design: for the vast majority of proteins
no ligand is known, and binding site identification methods can provide valuable 
information for understanding protein function, guiding rational drug design or 
identifying a protein as a potential drug target.
For both approaches, deep learning methods, and specifically
geometric deep learning have brought significant 
advances \citep{Gainza2020, Sverrisson2021, Mendez2021geometric, Ganea2022, Stark2022,Lu2022,Corso2023}.

\begin{figure*}
    \centering
    \includegraphics[width=1.0\textwidth]{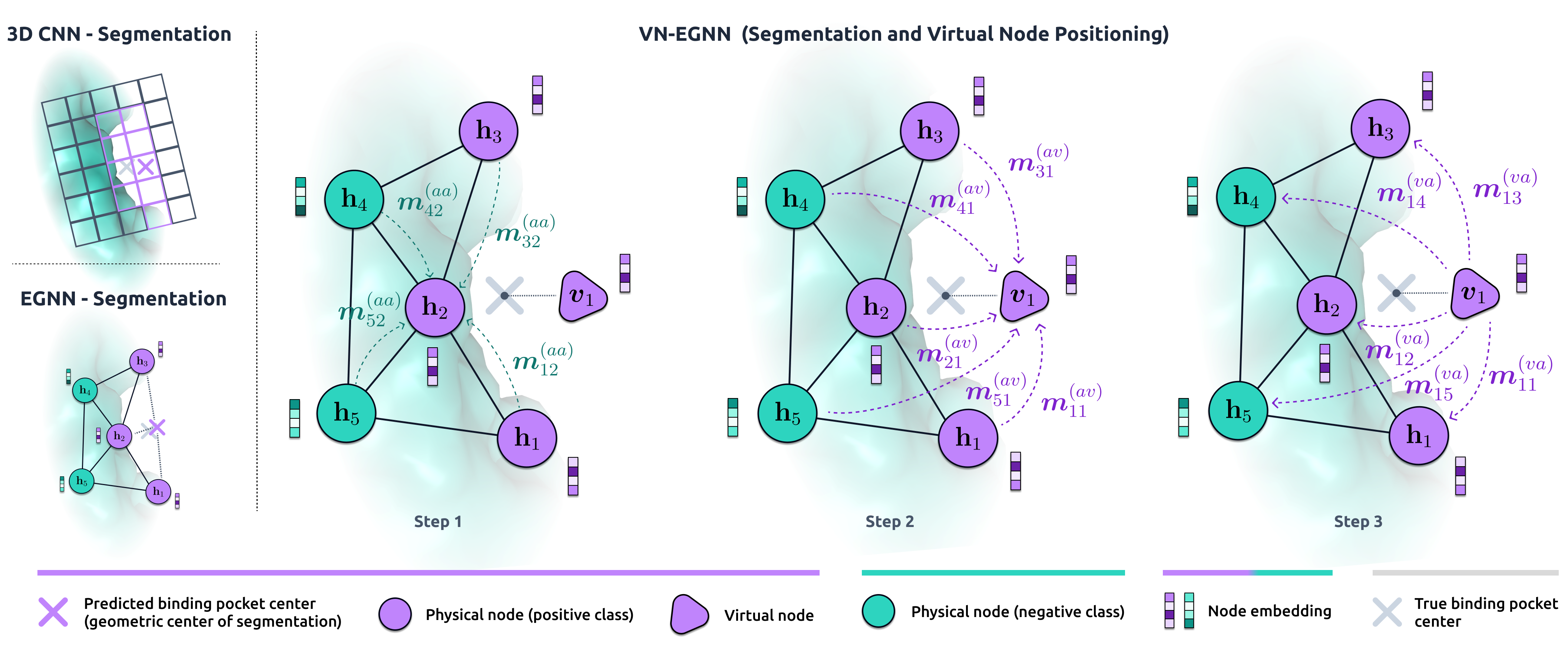}
    \caption{Overview of binding site identification methods. 
    \textbf{Top Left}: Traditional methods, based on segmentation of a voxel grid, in which 
    the pocket center is calculated as the geometric center of the positively labeled voxels. 
    \textbf{Bottom Left:} Geometric Deep Learning approaches, such as EGNN, 
    in which the pocket center is calculated as the geometric center of the positively labeled nodes. 
    \textbf{Right:} VN-EGNN approach (ours): 
    the predicted binding site center is the position of the virtual node after $L$ message passing layers.}
    \label{fig:overview}
\end{figure*}

\textbf{Methods for Binding Site Identification.} 
The identification of binding sites 
relies on the successful combination of physical, chemical and geometric information. 
Initially, machine learning methods for binding site prediction were based on carefully
designed input features due to their tabular processing structure. 
For instance, FPocket \citep{LeGuilloux2009} relies on Voronoi tessellation 
and alpha spheres \citep{Liang1998} and additionally takes an electronegativity
criterion into account. A random forest based method, 
which makes use of the protein surface, is P2Rank \citep{krivak2018p2rank}. 
With the advent of end-to-end deep learning and especially with the 
breakthrough of convolutional \citep{Lecun1998} and 
graph neural networks (GNNs) \citep{Scarselli2009, Defferrard2016, Kipf2017,  gilmer2017neural,Satorras2021}, 
the construction of input features can be learned which helped to advance predictive quality. 
For instance, DeepSite \citep{Jimenez2017} is a voxel-based 3D convolutional neural network for binding site prediction. 
Convolutional operations on the 3D space are, however, computationally very demanding
and so quickly other approaches to tackle binding site identification were developed, 
e.g., DeepSurf \citep{Mylonas2021} or PointSite \citep{Yan2022}. 
DeepSurf operates on surface-based representations and places several 
voxelized grids on the protein's surface, while PointSite is based on 
a form of sparse convolutions to reduce the computational overhead and 
keep sparse regions in the 3D space to be sparse. 
Typical convolutional networks, however, do not perform well at
binding site identification, likely because of the 
irregularity of protein structures and
due to the fact that proteins may be 
arbitrarily rotated and shifted in space \citep{zhang2023equipocket}. 
Thus, geometric deep learning approaches \citep{bronstein2021geometric}, most notably (graph-based) group equivariant 
architectures, such as EquiPocket \citep{zhang2023equipocket}, 
which are equivariant to the group of Euclidean transformations in the 3D space (E($3$)), should be powerful methods
for binding site identification. 

\textbf{E($3$)-equivariant Graph Neural Networks.}
We use graph neural networks (GNNs) that are robust to transformations of the Euclidean group, 
i.e., rotations, reflections, and translations, as well as to permutations.  
From a technical point of view, equivariance 
of a function $f$ to certain transformations means that 
for any transformation parameter $g$ and all inputs $x$ we have
$f(T_g(x)) = S_g(f(x)) \ ,\label{eq:equivariance_main}$
where $T_g$ and $S_g$ denote transformations on the input and output domain of $f$, respectively (see \Cref{eqBackground} for further information).
Equivariant operators applied to molecular graphs allow to preserve the geometric structure of the molecule. 
We build on E($n$)-equivariant GNNs (EGNNs) of~\citet{Satorras2021} applied to the three dimensional space
and the problem of binding site identification. 
In contrast to methods such as MACE~\citep{Batatia2022}, 
Nequip~\citep{Batzner2022}, PaiNN \citep{schutt2021equivariant} or SEGNN~\citep{brandstetter2021geometric}, 
EGNNs operate on scalar features, e.g., distances, and use scale operations for coordinate updates.
Thus, EGNNs operate efficiently~\citep{villar2021scalars} without resorting to compute-expensive higher order features, 
and, most importantly, allow for efficient coordinate update of virtual nodes.

\textbf{Limitations of GNNs and a mitigation strategy.}
Graph neural networks 
can suffer from limited 
expressiveness \citep{Morris2019,xu2019powerful}, 
oversmoothing \citep{Li2018, Rusch2023survey}, 
or oversquashing \citep{Alon2021, Topping2022}, which can lead to
unfavourable learning dynamics or weak predictive performance. 
To improve the learning dynamics
of GNNs, several works have introduced virtual nodes, 
sometimes called super-nodes or supersource-nodes, 
that are introduced into a message-passing scheme and 
connected to all other nodes. 
In a benchmark setting, \citet{Hu2020} showed that adding virtual nodes tends 
to increase the predictive performance. \citet{hwang2022an} provide a theoretical 
analysis of the benefits of virtual nodes in terms of expressiveness, 
demonstrate the increased expressiveness of GNNs with virtual nodes and 
also hint at the fact that such nodes can decrease oversmoothing. 
\citet{Alon2021} mention that virtual nodes might 
be used as a technique to overcome oversquashing effects. 
\citet{cai2023connection} and \citet{cai2023local} show that an MPNN 
with one virtual node, connected to all nodes, can approximate a Transformer layer. 
Low rank global attention \citep{puny2020graph} can be seen as one virtual node, 
which improves expressiveness.
Practically, virtual nodes have already been suggested
in the original work by \citet{gilmer2017neural}, and 
they were even  mentioned earlier in \citet{Scarselli2009} 
and used in application areas 
such as drug discovery \citep{li2017learning,pham2017graph,ishiguro2019graph}.
\citet{Joshi2023} investigated the expressive power of EGNNs in greater detail 
and argue that these networks can suffer from oversquashing. 
In order to alleviate the oversquashing problem of EGNNs for binding site identification, 
we suggest to extend EGNNs with virtual nodes 
and introduce an adapted message passing scheme. 
We refer to this method as Virtual-Node Equivariant GNN (VN-EGNN). 
To the best of our knowledge, VN-EGNN is the first 
E($3$)-equivariant GNN  architecture using virtual nodes (and thereby not relying on a priori knowledge).

\textbf{EGNNs with virtual nodes for binding site identification.}
In accordance with previous approaches, 
we consider binding site identification as a segmentation task. 
While other methods are, e.g., based on voxel grids, 
our method is based on EGNNs with virtual nodes, 
where all atoms or residues of the protein 
(the physical entities) 
are represented by physical, i.e., a non-virtual, nodes in the graph (see \Cref{fig:overview}, left). 
The objective is to correctly classify whether a node is 
within a certain radius of a region to which potential ligands can bind. 
Therefore, binding site identification can be considered as a 
node-level binary classification task and thus a semantic segmentation task. 
For this task, the ground truth
is whether an atom was within a 
certain radius to experimentally observed protein binding ligands.
In addition to node features, EGNNs act on coordinate features associated with each node, and both feature types are updated during message passing.
While it appears straightforward to associate physical nodes with the protein's atoms, 
it is a-priori unclear if the coordinate embeddings of virtual nodes are useful
for the task at hand. 
In an initial experiment, we trained VN-EGNN to learn a semantic segmentation task 
using multiple virtual nodes to which we assigned random coordinates. 
In an analysis of the results, we could empirically observe that coordinates of the 
virtual nodes converged towards the actual physical binding positions of 
ligands on the protein (see~\cref{appsec:intial} for further details and \Cref{fig:overview}, 
right, for a visualization).
The results of this initial experiment gave rise to the assumption that virtual nodes 
enable VN-EGNNs to form useful neural representations of binding sites 
and especially allow to predict locations of binding site centers. 
This, however, further implies that the binding site center itself may be a useful optimization target to train VN-EGNNs and, thus, we extended our objective 
to not only correctly predict whether physical nodes are close to binding regions, 
but also to directly take the distance between observed and predicted binding site centers 
into account. With this multi-modal objective, the coordinate embeddings of virtual nodes are trained to predict 
the locations of binding site centers. The remaining features of the virtual 
nodes are considered to form an abstract neural representation of a protein binding site.

\textbf{Contributions.}
In this work, we aim at improving binding site identification through geometric
deep learning methods. Here, we follow the approach of using EGNNs \citep{Satorras2021,zhang2023equipocket}
for identifying binding pockets. 
Although EGNNs are prime candidates for this task, they 
exhibit poor performance at binding site identification \citep{zhang2023equipocket},
which might be due to 
a) their lack of dedicated nodes that can learn representations of binding sites, and 
b) oversquashing effects which hamper learning \citep{Alon2021, Topping2022,Joshi2023}.
We aim at alleviating both problems by extending EGNNs with 
so-called virtual nodes. 
In this work, we contribute the following:\\
\textbf{a)} We propose a novel type of graph neural network geared towards the identification of binding sites of proteins.\\
\textbf{b)} We demonstrate that the virtual nodes in the message-passing scheme learn useful representations and accurate locations 
of binding pockets.\\
\textbf{c)} We assess the performance of other methods, baselines and our method on benchmarking datasets.

\section{E($3$)-Equivariant Graph Neural Networks with Virtual Nodes}
\label{Method}

\subsection{Notational Preliminaries}
\label{subsec:notation}

We give an overview on variable and symbol notation in \Cref{notation},
and a more detailed description and discussion on how we 
represent proteins and binding sites in \Cref{protein,bindingSite} 
respectively. 
To quickly summarize, the coordinates of 
the $i$-th \emph{physical node}, 
e.g., the location of an atom of a protein,
are denoted as $\Rx_i \in \mathbb{R}^3$, and 
its other node features as $\Bh_i \in \mathbb{R}^D$. $l$ added as an index to symbols will indicate neural network layers, but might be omitted for simplicity sometimes if it is clear from the context.
We will consider \emph{virtual nodes} and the $k$-th virtual 
node coordinates will be denoted as $\Bz_k$, while the other virtual node features will be denoted as $\Bv_k$.
We use upper-case bold letters to denote the matrices
collecting the coordinates and features of the $N$ 
physical and the $K$ virtual nodes, respectively:
$\BH^l:=(\Bh_1^l,\ldots,\Bh_N^l)$,
$\RX^l:=(\Rx_1^l,\ldots,\Rx_N^l)$,
$\BV^l:=(\Bv^ {l}_1,\ldots,\Bv^{l}_K)$,
$\RZ^l:=(\Rz^{l}_1,\ldots,\Rz^{l}_K)$.

We denote the graph, which VN-EGNN works upon as $\mathcal G$ and 
$\BA$ as the associated adjacency matrix. 
$\mathcal N(i)$ indicates neighbouring 
nodes to node $i$ within $\mathcal G$ 
and edge features between nodes $i$ and $j$ as $a_{ij}$ (for two non-virtual nodes) and $d_{ij}$ (if $j$ is virtual). 
At training time, we might have access to node-level labels
$y_i=1$ and $y_i=0$. 
We assume that in the training set we have access to 
a set of $M$ center coordinates $\{\Ry_m\}_{m=1}^M$
with $\Ry_m \in \mathbb{R}^3$, which the model should predict. 
We denote predicted node-level 
labels by $\hat y_i \in [0,1]$ and the set of
$K$ predicted center coordinates from the model
by $\{\hat \Ry_k\}_{k=1}^K$ with
$\hat \Ry_k \in \mathbb{R}^3$. 

\subsection{EGNNs and their application to 
Binding Site Identification}
\label{sec:egnn_bsi}
EGNNs are straightforward to apply to proteins, when they are 
represented by a neighborhood graph $\mathcal P$,
in which each node represents an atom
and edges between two atoms represent 
spatially close atoms (distance between the atoms in the protein is below some threshold). To apply EGNNs, we first set $\mathcal G=\mathcal P$. 
For binding pocket identification, one could 
predict node labels $y_i$, which indicate
whether the atom belongs to a binding pocket or not. 

The physical nodes represent atoms and their 
initial coordinate features are set to the location 
of the atoms $\Rx^0_i$,
and the initial node features $\Bh^0_i$
to, e.g., the atom or residue type.
Then we apply the layer-wise message passing scheme $\left(\RX^{l+1}, \BH^{l+1}\right) = 
            \mathrm{EGNN}\left( \RX^{l}, \BH^{l}, \BA \right)$ (\crefrange{eq:Std_message1}{eq:Std_update1}) as given by \citet{Satorras2021}:
\begin{align}
    &\Bm_{ij} = \Bphi_{e}(\Bh_i^l,\Bh_j^l,\lVert \Rx_i^l - \Rx_j^l \rVert^2, a_{ij}) \label{eq:Std_message1} \\
    &\Bm_i = \sum_{j\in \mathcal N(i)} \Bm_{ij} \label{eq:Std_sum1}\\  
    &\Rx_i^{l+1} = \Rx_i^l + \frac{1}{|\mathcal N(i)|} \sum_{j\in \mathcal N(i)} \frac{\Rx_i^l - \Rx_j^l}{\lVert \Rx_i^l - \Rx_j^l\rVert} \phi_{x}(\Bm_{ij}) \label{eq:Std_x_coord1} \\ 
    &\Bh_i^{l+1} = \Bphi_{h}\left(\Bh_i^l, \Bm_i\right),\label{eq:Std_update1} 
\end{align}
where $\Bphi_{e}$, $\phi_{\Rx}$ and $\Bphi_{h}$ 
denote multilayer-perceptrons (MLPs). 
To identify binding pockets, 
we can extract predictions $\hat y_i$ for
each atom $i$ by a read-out 
function applied to the last message passing step $L$, i.e., 
we have: $\hat y_i = \sigma ( \boldsymbol w^\top \Bh_i^L)$ with 
an activation function $\sigma$ and parameters $\Bw$. Our model does not 
incorporate edge features, symbolized by $a_{ij}$. Hence, we will exclude 
these from the subsequent EGNN formulations and their related derivations.

\subsection{VN-EGNN: Extension of EGNN with virtual nodes}
\label{subsec:vn_message_passing}
We now extend $\mathcal G$ (which is set to the protein neighborhood graph $\mathcal P$ for the task of protein binding site identification)
with a set of $K$ virtual nodes, 
which exhibit edges to all other nodes, which will allow
us to learn representations of hidden geometric entities, 
such as binding sites, and simultaneously ameliorate
oversquashing. To be able to process this extended graph,
we modify EGNNs by locating the virtual nodes 
at coordinates $\RZ=(\Rz_1,\ldots,\Rz_K) \in \mathbb{R}^3$ and 
associating them with 
a set of properties $\BV=(\Bv_1,\ldots,\Bv_K) \in \mathbb{R}^D$. 
The new message passing scheme $\left( \RX^{l+1},\BH^{l+1},\RZ^{l+1}, \BV^{l+1} \right) =  \;\;\mathrm{VN\text{-}EGNN}  \left(\RX^{l},\BH^{l},\RZ^{l}, \BV^{l},\BA \right)$
of a single VN-EGNN layer consists of three phases (\crefrange{eq:message1}{eq:update1}, \crefrange{eq:message2}{eq:update2}, and, \crefrange{eq:message3}{eq:update3}), in which the feature and coordinate embeddings of the physical nodes are updated twice:
\begin{align}
\Bh^l_i \rightarrow \Bh^{l+\sfrac{1}{2}}_i \rightarrow \Bh^{l+1}_i, \quad  \Rx^l_i \rightarrow \Rx^{l+\sfrac{1}{2}}_i \rightarrow \Rx^{l+1}_i \quad  \forall i 
\end{align}
while 
virtual node embeddings are only updated once per message passing step
\begin{align}
    \Bv^l_k \rightarrow \Bv^{l+1}_k, \quad  \Rz^l_k\rightarrow \Rz^{l+1}_k  \quad \forall k.
\end{align}

\textbf{Message Passing Phase I} between \textbf{physical nodes} (analogous to EGNN):
    \begin{align}
        &\Bm_{ij}^{(aa)} = \Bphi_{e^{(aa)}}(\Bh_i^l,\Bh_j^l,\lVert \Rx_i^l - \Rx_j^l \rVert, a_{ij}) \label{eq:message1} \\ 
        &\Bm_i^{(aa)} = \frac{1}{|\mathcal N(i)|} \sum_{j\in \mathcal N(i)} \Bm_{ij}^{(aa)} \label{eq:sum1}\\  
        &\Rx_i^{l+\sfrac{1}{2}} = \Rx_i^l + \frac{1}{|\mathcal N(i)|} \sum_{j\in \mathcal N(i)} \frac{\Rx_i^l - \Rx_j^l}{\lVert \Rx_i^l - \Rx_j^l\rVert} \phi_{x^{aa}}(\Bm_{ij}^{(aa)}) \label{eq:x_coord1}  \\
        &\Bh_i^{l+\sfrac{1}{2}} = \Bh_i^l + \Bphi_{h^{(aa)}}\left(\Bh_i^l, \Bm_i^{(aa)}\right). \label{eq:update1} 
    \end{align}

\textbf{Message Passing Phase II} from \textbf{physical nodes} to \textbf{virtual nodes}:
    \begin{align}
        &\Bm_{ij}^{(av)} = \Bphi_{e^{(av)}}(\Bh_i^{l+\sfrac{1}{2}}, \Bv_j^l, \lVert \Rx_i^{l+\sfrac{1}{2}} - \Rz_j^l \rVert, d_{ij}) \label{eq:message2}
        \\ &\Bm_j^{(av)} = \frac{1}{N} \sum_{i=1}^N \Bm_{ij}^{(av)} \label{eq:sum2} \\  
        &\Rz_j^{l+1} = \Rz_j^l + \frac{1}{N} \sum_{i=1}^N \frac{\Rx_i^{l+\sfrac{1}{2}} -\Rz_j^l}{\lVert \Rx_i^{l+\sfrac{1}{2}} - \Rz_j^l \rVert} \phi_{x^{av}}(\Bm_{ij}^{(av)})  \label{eq:z_coord}
        \\ &\Bv_j^{l+1} = \Bv_j^l + \Bphi_{h^{(av)}}\left(\Bv_j^l, \Bm_j^{(av)}\right) \label{eq:update2} 
    \end{align}

\textbf{Message Passing Phase III} from \textbf{virtual node} to \textbf{physical nodes}:
    \begin{align}
        &\Bm_{ij}^{(va)} = \Bphi_{e^{(va)}}(\Bv_i^{l+1},\Bh_j^{l+\sfrac{1}{2}},\lVert \Rz_i^{l+1} - \Rx_j^{l+\sfrac{1}{2}}\rVert, d_{ji}) \label{eq:message3} 
        \\ &\Bm_j^{(va)} = \frac{1}{K} \sum_{i=1}^K \Bm_{ij}^{(va)} \label{eq:sum3} \\ 
            &\Rx_j^{l+1} = \Rx_j^{l+\sfrac{1}{2}} + %
            \frac{1}{K} \sum_{i=1}^K \frac{\Rz_i^{l+1} - \Rx_j^{l+\sfrac{1}{2}}}{\lVert \Rz_i^{l+1} - \Rx_j^{l+\sfrac{1}{2}} \rVert} \phi_{x^{va}}(\Bm_{ij}^{(va)})
        \label{eq:x_coord2}
        \\  &\Bh_j^{l+1} = \Bh_j^{l+\sfrac{1}{2}} + \Bphi_{h^{(va)}}\left(\Bh_j^{l+\sfrac{1}{2}}, \Bm_j^{(va)}\right) \label{eq:update3} %
    \end{align}

Here,  $\Bphi_{e^{(aa)}}, \ldots, \Bphi_{h^{(va)}}$ are again MLPs.  
The MLPs $\phi_{.}$ are layer-specific, i.e. $\phi^l_{.}$ and 
currently do not consider edge features $d_{ij}$ and $a_{ij}$.
To keep the notation uncluttered, 
we skipped the layer index $l$ for the MLPs in the above formulae.

\textbf{Initialization of virtual nodes in VN-EGNN.}
The $K$ virtual nodes are initially evenly distributed across 
a sphere using a Fibonacci grid (see \Cref{appsec:fibonacci}), 
of which the radius 
is defined as the distance between the protein center and its most distant atom.
We initialize the virtual node properties 
$\Bv^0_k$ by averaging over the initial features $\Bh^0_i$. 

\subsection{Properties of VN-EGNN}
\label{sec:properties_vnegnn}
The following proposition shows, that analogously to EGNNs, 
VN-EGNNs are equivariant  
with respect to roto-translations and reflections by construction.

\begin{proposition}
    \label{prop:equivariance}
    E$(3)$-equivariant graph neural networks with virtual nodes as defined in 
    \crefrange{eq:message1}{eq:update3} are equivariant with respect 
    to roto-translations and reflections
    of the input and virtual node coordinates.
\end{proposition}

\begin{proof}
 See \Cref{appsec:proof}.
\end{proof}

\paragraph{Invariance with respect to the initial coordinates of the virtual nodes.} 
Note that \Cref{prop:equivariance} aims for equivariance with 
respect to rotations of the physical protein nodes $\RX$ and arbitrary, but fixed initialized virtual nodes $\RZ$ . We further want to have predictions, which are approximately invariant to differently chosen initial virtual node coordinates $\RZ^0$. This ultimately leads to predictions that are approximately equivariant with respect to E($3$)-transformations of the physical protein nodes.
In practice, we distribute the initial virtual node coordinates evenly on a sphere according to an algorithm, which constructs a spherical Fibonacci grid \citep{Swinbank2006}. The algorithm provides spherically distributed grid points, which are fixed at 
certain locations in the 3D space. In order to achieve invariance 
with respect to differently chosen initial virtual node 
coordinates, we randomly rotate this grid of initial virtual node 
coordinates for each sample in every epoch, i.e. there is 
variation in the relative alignment of the Fibonacci grid points, 
that represent the virtual node positions, to physical protein 
nodes. Empirically, we observe that this training strategy leads to approximate invariance to different initializations of the virtual node coordinates (see \Cref{tab:varRot}).

\paragraph{A potential alternative strategy for initialization 
virtual node coordinates.} 
An idea to avoid random alignments between physical node 
coordinates and initial virtual 
node coordinates, would be, to change initial coordinates in an 
equivariant way with respect to E($3$) 
group transformations of the protein 
physical nodes. This could be achieved, e.g., by defining frames 
\citep{puny2022frame} based on Principal Component Analysis of 
physical protein node coordinates and by aligning the Fibonacci 
grid relative to these frames. Consequently,
we would achieve that 
binding pocket predictions would change equivariantly with E($3$)-
transformations of the protein. Thereby the definition of such 
frames via Principal Component Analysis (PCA)
is possible up to certain degenerate cases, that 
occur with probability zero for proteins. Since the orientation 
of axes might still not be unique, a strategy might be to 
compute properties such as the overall molecular weight for each 
octant in the coordinate system spanned by PCA eigenvectors. The 
orientation can then be set, such that for the octant with the 
maximum overall molecular weight, all coordinates get positive 
values.

\textbf{Virtual nodes ameliorate oversquashing by 
bounding the maximal shortest-path distance between nodes and required message-passing steps.}
Several works \citep{Alon2021, Topping2022, di2023over} have 
investigated the relation between oversquashing and characteristics
of the MPNN layers and the adjacency matrix. According to
\citet{Topping2022} oversquashing is defined as 
$\frac{\partial h_i^{r+1}}{\partial h_j^{0}}$,
which is the effect of that one node with index $i$ has 
on a node with index $j$ during learning, where the nodes
are at a shortest-path distance of $r+1$.  Critically, this quantity
can be bounded by the model parameters of the involved
MLPs and the topology of the graph \citep{di2023over}, 
concretely the normalized adjacency matrix. 
We use \citet[Lemma~1]{Topping2022}, which states that
$\left| \frac{\partial h_i^{r+1}}{\partial h_j^{0}} \right| \leq (\alpha \beta)^{r+1} (\hat \BA^{r+1})_{ij}$
where $\alpha$ and $\beta$ are bounds on the element-wise gradients 
of the MLPs of the message-passing network,
$h_i^{r+1}$ is one component of the node representation 
of node $i$ in message passing layer $r+1$.
The quantity $r+1$ is both the number of message-passing layers and 
the shortest-path distance of nodes $i$ and $j$ 
in the graph, and $\hat \BA$ is the normalized 
adjacency matrix, for which the diagonal values of the original matrix 
is set to $1$. The normalized adjacency matrix $\hat \BA$
is a symmetric positive matrix that
has a leading eigenvalue at $1$ \citep{perron1907zur,frobenius1912matrizen}, 
such all other 
the eigenvectors of all other 
eigenvalues of $\hat \BA^r$ decay exponentially
with $r$. Depending on the weights and activation functions of
the MLPs, $|\alpha \beta|^{r+1}$ either grows or vanishes exponentially
with $r$, which might lead to either exploding or vanishing gradients, respectively. 
Thus, learning can only be stabilized via keeping $r$ stable, which
virtual nodes that are connected to all other nodes can provide
since they bound both the maximal path distance and the necessary
number of message-passing steps by $r+1=2$ .

\paragraph{Expressiveness of VN-EGNN.}
The expressive power of GNNs is linked to their ability to distinguish non-isomorphic graphs. While a minimum of $k$ layers of an EGNN is required to distinguish two $k$-hop distinct graphs, one layer of VN-EGNN is presumed to be sufficient, as can be shown by the application of the Geometric Weisfeiler-Leman test, which serves as an upper bound on the expressiveness of EGNNs. Experimental findings on $k$-chain geometric graphs support this proposition and demonstrate the increased expressive power of VN-EGNN compared to EGNNs without virtual nodes. For further details and an empirical study see~\cref{appsec:expressivity}.

\subsection{Training VN-EGNNs}
\textbf{Objective.}  
Previous methods, which consider binding site identification
as a node-level prediction task (see \cref{sec:egnn_bsi})
$\hat y_n = \sigma ( \boldsymbol w^\top \Bh_n^L)$, 
use a type of \emph{segmentation loss}. The segmentation loss
can be either the cross-entropy loss $\mathrm{CE}$:
\[\mathcal L_{\mathrm{segm}} = \frac{1}{N}\sum_{n=1}^N \mathrm{CE}(y_n,\hat y_n)\]
or the Dice loss, that is based on the continuous Dice coefficient \citep{Shamir2018}, with $\epsilon=1$:
\[\mathcal L_{\mathrm{dice}} \coloneqq 1 - \frac{2 \ \sum_{n=1}^N y_n \ \hat y_n + \epsilon}{\sum_{n=1}^N y_n + \sum_{n=1}^N \hat y_n + \epsilon}.\]

Our introduction of virtual nodes with coordinates allows to directly 
tackle the much more challenging problem to predict binding site region center 
points and to extract predictions for these points as outputs of the last EGNN layer.
For each protein in the training set, we know 
the geometric center of the binding site, that
we denote as $\{\mathbf{y}_1, \ldots, \mathbf{y}_M\}$.
The read-out $\hat \Ry_k$ for each virtual node
$\hat \Ry_k:=\Rz_k^L$ ($1 \leq i \leq K$)
is just the coordinate embedding $\Rz_k^L$ in 
the last layer $L$. Each known binding site center should be
detected by at least one virtual node, via its read-out, 
which leads to the following objective
\begin{align}
\mathcal L_{\mathrm{bsc}} = \frac{1}{M} \sum_{m=1}^M
\min_{k \in {1,\ldots,K}} \lVert \Ry_m - \hat \Ry_k \rVert^2.
\end{align}

The full objective of VN-EGNN for binding site
identification is 
\[\mathcal L=\mathcal L_{\mathrm{bsc}}+\mathcal L_{\mathrm{dice}},\]
in which the two terms could also be balanced against
each other through a hyperparameter, 
which we found was not necessary though.

\textbf{Self-Confidence Module.}
We employ a self-confidence module \citep{Jumper2021, zhang_e3bind_2023}, 
to assess the quality of predicted binding sites, by equipping each prediction with a confidence score. This allows a ranking of the predictions, similar to \citet{krivak2018p2rank}. The confidence value, indicated by $\hat{c}_k$, is computed through $\hat{c}_k = \psi(\Bv_k)$, with $\psi$ implemented as an MLP. 
During training, the target values for the confidence prediction are generated on-the-fly from the 
predicted positions $\hat \Ry_i$ and the closest known binding pocket
center $\Ry_m$, in analogy with confidence scores for object 
detection methods in computer vision.

The confidence label for the $k$-th virtual node is obtained defined by \citep{zhang_e3bind_2023}:
\begin{align}
    c_k = 
\begin{cases} 
1 - \frac{1}{2\gamma} \cdot \lVert \Ry_k - \hat \Ry_k \rVert & \text{if }\lVert \Ry_k - \hat{\Ry}_k \rVert \leq \gamma, \\
c_0 & \text{otherwise}
\end{cases},
\label{eq:confidence}
\end{align}
with $c_0=0.001$, 
To align with the commonly accepted threshold value 
for the DCC/DCA success rates of 4\r{A} we choose $\gamma = 4$. The loss on the confidence score is squared loss:
\begin{align}
    \mathcal{L}_{\text{confidence}} = \frac{1}{K}\sum_{k=1}^K (c_k-\hat{c}_k)^2.
\end{align}

\definecolor{light-gray}{gray}{0.0}

\begin{table*}[tb] 
\begin{adjustbox}{max width=\textwidth}

\begin{threeparttable}
        \setlength{\abovecaptionskip}{0cm} 
        \setlength{\belowcaptionskip}{0cm}
\caption{Performance at binding site identification in 
terms of DCC and DCA success rates.\tnote{a}
The first column provides the method, the second the 
number of parameters of the model, the fourth and the 
fifth column the performance on the COACH420 dataset, 
the sixth and seventh column the performance on the 
HOLO4K dataset, and the remaining columns the performance
on PDBbind 2020. The best performing method(s)
per column are marked bold. The second best in italics.}
\label{tab:results}
\centering
\begin{tabular}{l c c c c c c c c c c c c}
\toprule
 \multirow{2}{*}{Methods}       & Param  &  \multicolumn{2}{c}{COACH420} & \multicolumn{2}{c}{HOLO4K\tnote{d}} & \multicolumn{2}{c}{PDBbind2020}  \\    
 \cmidrule(r){3-4}  \cmidrule(r){5-6}  \cmidrule(r){7-8} 
                              &    (M)           &   DCC$\uparrow$          &     DCA$\uparrow$   &   DCC$\uparrow$      &   DCA$\uparrow$ &  DCC$\uparrow$   &  DCA$\uparrow$     \\     \midrule
Fpocket \citep{LeGuilloux2009}\tnote{b}          &  \textbackslash  &  0.228  & 0.444     & 0.192    &  0.457    &  0.253   & 0.371    \\
\color{light-gray} P2Rank \citep{krivak2018p2rank}\tnote{c}         & \color{light-gray}\textbackslash  &  \color{light-gray}\emph{0.464}  & \color{light-gray}\emph{0.728}   & \color{light-gray} \emph{0.474}   &   \color{light-gray}\textbf{0.787}     &  \color{light-gray}\emph{0.653}  &  \color{light-gray}\textbf{0.826}    \\ \midrule
DeepSite \citep{Jimenez2017}\tnote{b}            & 1.00   &   \textbackslash    &    0.564    &  \textbackslash   &  0.456    & \textbackslash     & \textbackslash      \\ 
Kalasanty \citep{stepniewska2020improving}\tnote{b}                                & 70.64  & 0.335   &  0.636      &  0.244   & 0.515    &  0.416   & 0.625   \\ 
DeepSurf \citep{Mylonas2021}\tnote{b}                                   & 33.06  & 0.386    & 0.658      &  0.289  &  0.635   &  0.510   & 0.708    \\ 
DeepPocket \citep{Aggarwal2022}\tnote{c}                                   & \textbackslash  & 0.399    & 0.645      &  0.456  &  0.734  &  0.644  &   \textbf{0.813}  \\
\midrule
GAT \citep{velivckovic2017graph}\tnote{b}         & \textbf{0.03} & 0.039(0.005)    & 0.130(0.009)   & 0.036(0.003)    & 0.110(0.010)    & 0.032(0.001)    & 0.088(0.011)     \\ 
GCN \citep{Kipf2017}\tnote{b}                    &  0.06  &  0.049(0.001)  & 0.139(0.010)       &  0.044(0.003)    & 0.174(0.003)    &   0.018(0.001) & 0.070(0.002)       \\
GAT + GCN\tnote{b}                    &  0.08   & 0.036(0.009)  &  0.131(0.021)     &  0.042(0.003)   & 0.152(0.020)     &  0.022(0.008)   &  0.074(0.007)      \\
GCN2 \citep{chen2020simple}\tnote{b}                                      & 0.11  & 0.042(0.098)   &  0.131(0.017)      & 0.051(0.004)    & 0.163(0.008)     & 0.023(0.007)    &  0.089(0.013)       \\ \midrule
SchNet \citep{schutt2017schnet}\tnote{b}        & 0.49    & 0.168(0.019) & 0.444(0.020)   &  0.192(0.005)    & 0.501(0.004)    &  0.263(0.003)    &  0.457(0.004)         \\ 
EGNN \citep{Satorras2021}\tnote{b}    &   0.41  & 0.156(0.017)   &  0.361(0.020)      &  0.127(0.005)   & 0.406(0.004)    & 0.143(0.007)    & 0.302(0.006)       \\ \midrule
EquiPocket \citep{zhang2023equipocket}\tnote{b}        & 1.70  & 0.423(0.014)  &  0.656(0.007)    & 0.337(0.006)  & \emph{0.662}(0.007)   &  0.545(0.010)  &  0.721(0.004)      \\ \midrule
VN-EGNN (ours)   & 1.20 & \textbf{0.605(0.009)} & \textbf{0.750(0.008)} & \textbf{0.532(0.021)} & 0.659(0.026) & \textbf{0.669(0.015)} & \textbf{0.820(0.010)}\\
\bottomrule

\end{tabular}

  \begin{tablenotes}
    \item[a] The standard deviation across training re-runs is indicated in parentheses.
    \item[b] Results from \citet{zhang2023equipocket}.
    \item[c] Uses different training set and, thus,
    limited comparability. 
    \item[d] This dataset represents a strong domain shift from the training data for all methods (except for P2Rank). Details on the domain shift 
    in~\cref{appsec:holo4k_artifacts}.
 \end{tablenotes}
\end{threeparttable}
\end{adjustbox}
\end{table*}

\section{Experiments}
\label{sec:experiments} \label{Experiments}

\subsection{Data}
We use the benchmarking setting of \citet{zhang2023equipocket} performing experiments on four datasets 
for binding site identification:
\textbf{scPDB} \citep{desaphy2015sc}, 
\textbf{PDBbind} \citep{2004The}, 
\textbf{COACH420} and \textbf{HOLO4K}.
For details, see~\cref{appsec:experimental_settings}.

\begin{table*}
\begin{adjustbox}{max width=\textwidth}
\begin{threeparttable}
        \setlength{\abovecaptionskip}{0cm} 
        \setlength{\belowcaptionskip}{0cm}
\caption{Ablation study. The main components of the VN-EGNN architecture are ablated and tested for their performance on 
the benchmarking datasets. The first column reports the 
variant of the ablated method, the second column whether 
the method contains virtual nodes (VN), the third column
whether the method applies heterogenous message-passing, 
and the fourth column whether ESM embeddings were used. The 
remaining columns are analogous to~\cref{tab:results}.
}
\label{tab:results_ablation}
\centering
\begin{tabular}{l c c c c c c c c c c c c c c}
\toprule
 \multirow{2}{*}{Methods}       & \multirow{2}{*}{VN } & heterog.  & \multirow{2}{*}{ESM}  &  \multicolumn{2}{c}{COACH420} & \multicolumn{2}{c}{HOLO4K} & \multicolumn{2}{c}{PDBbind2020}  \\    
 \cmidrule(r){5-6}  \cmidrule(r){7-8}  \cmidrule(r){9-10} 
                         & &  MP   &              &   DCC$\uparrow$          &     DCA$\uparrow$   &   DCC$\uparrow$      &   DCA$\uparrow$ &  DCC$\uparrow$   &  DCA$\uparrow$     \\     \midrule
EGNN \citep{Satorras2021}\tnote{b}    &   \xmark & \xmark &\xmark  & 0.156(0.017)   &  0.361(0.020)      &  0.127(0.005)   & 0.406(0.004)    & 0.143(0.007)    & 0.302(0.006)       \\ 
VN-EGNN (residue emb.)    & \checkmark & \checkmark & \xmark& 0.503(0.022) & 0.684(0.016) & 0.438(0.019) & 0.605(0.013) & 0.551(0.017) & 0.751(0.009)\\
VN-EGNN (homog.)    & \checkmark & \xmark & \xmark& 0.497(0.014) & 0.700(0.013) & 0.414(0.023) & 0.618(0.024) & 0.502(0.029) & 0.717(0.025)\\
VN-EGNN (homog.)    & \checkmark & \xmark & \checkmark& 0.575(0.008) & 0.708(0.009) & 0.479(0.012) & 0.595(0.010) & 0.649(0.010) & 0.805(0.006)\\
VN-EGNN (full)    & \checkmark & \checkmark & \checkmark& \textbf{0.605(0.009)} & \textbf{0.750(0.008)} & \textbf{0.532(0.021)} & \textbf{0.659(0.026)}   & \textbf{0.669(0.015)} & \textbf{0.820(0.010)}\\

\bottomrule

\end{tabular}

  \begin{tablenotes}
    \item[a] The standard deviation across training re-runs is indicated in parentheses.
    \item[b] Results from \citet{zhang2023equipocket}. 
 \end{tablenotes}
\end{threeparttable}
\end{adjustbox}
\end{table*}

\subsection{Evaluation}

\textbf{Methods compared.} We compare the 
following binding site identification methods
from different categories: 
\emph{Geometry-based}: 
Fpocket \citep{LeGuilloux2009} and P2Rank \citep{krivak2018p2rank}. %
\emph{CNN-based}: 
DeepSite \citep{Jimenez2017}, 
Kalasanty \citep{stepniewska2020improving}, 
DeepPocket \citep{aggarwal2021deeppocket}
and
DeepSurf \citep{Mylonas2021}.
\emph{Topological graph-based}: 
GAT \citep{velivckovic2017graph}, 
GCN \citep{Kipf2017}, and
GCN2 \citep{chen2020simple}.
\emph{Spatial graph-based}: 
SchNet \citep{schutt2017schnet}, 
EGNN \citep{Satorras2021}, 
EquiPocket \citep{zhang2023equipocket}, 
and our proposed VN-EGNN.

\textbf{Evaluation metrics.}
We used the \emph{DCC/DCA success rate}, which are well-established metrics for binding site identification \citep[see e.g.,][]{chen2011critical}.
The \emph{DCC} is defined as the distance between the predicted and known binding site centers, whereas the \emph{DCA} is defined as the shortest distance between the predicted 
center and any atom of the ligand. Following  \citet{stepniewska2020improving} and \citet{zhang2023equipocket}, predictions within a certain threshold of DCC and DCA, are considered as successful, which is commonly referred to as DCC/DCA success rate. Adhering to these works, we maintained a threshold of 4\AA\ throughout our experiments (for other thresholds, see~\cref{fig:dcc_thresholds}).
In line with \citet{chen2011critical,zhang2023equipocket,stepniewska2020improving} for each protein only $M$ predicted binding sites with  the highest self-confidence scores $\hat{c}_k$ are considered, where $M$ is the number of known binding sites of the protein. Subsequently, each predicted binding site was aligned with the closest real binding site and DCC/DCA success rate was calculated.

\subsection{Implementation details.}
{\footnote{Code is available at \url{https://github.com/ml-jku/vnegnn}}
\label{subsec:impl_details}
We used AdamW \citep{loshchilov2017decoupled} optimizer for 1500 epochs, selecting the best checkpoint based on the validation dataset. We used 5 VN-EGNN layers, where each layer consists of the three step message passing scheme described in \Cref{subsec:vn_message_passing}, the feature and message size was set to 100, in all layers. 
Due to the possibility of different virtual nodes converging to identical locations, we employed Mean Shift Algorithm \citep{comaniciu2002mean}, to cluster virtual nodes that are in close spatial proximity. By averaging their self-confidence scores and positions, we treated these clustered nodes as a single instance. Because of the large complexes in HOLO4K, 
we ran VN-EGNN for each chain and 
merged the predicted pocket centers.
For the initial residue node features we used pre-trained ESM-2 \citep{Lin2022.07.20.500902} protein embeddings following \citet{Corso2023, pei2023fabind}. For virtual nodes, we derived their features by averaging the residue node features across the entire protein. We used the position of the $\alpha$-carbons as residue node locations. Virtual nodes are connected to residue nodes solely, but not with each other. A linear layer was used to map these initial features to the required dimensions ($\Bh^0_n$, $\Bv_k^0$) of the model. Layer normalization and Dropout \citep{10.5555/2627435.2670313} was applied in each message passing layer. SiLU \cite{hendrycks2016gaussian} activation was used across all layers.  Analogous to \citet{pei2023fabind} we applied normalization (divided by 5) and unnormalization (multiplied by 5) on the coordinates and used Huber loss \citep{Huber1964RobustEO} for the coordinates, which empirically proved to be slightly more effective. The learning rate was set to $10^{-3}$, after 100 epochs we reduced the learning rate by factor of $10^{-1}$ if the model did not improve for 10 epochs. For training we used 4x NVIDIA A100 40GB GPU with batch size set to 64 on each GPU. The training time was about 8 hours. 
Hyperparameters were selected based on a validation dataset where we spilt 10\% of the training data (see \Cref{tab:hyperparams}).

\begin{figure*}
    \centering
    \includegraphics[width=0.90\columnwidth]{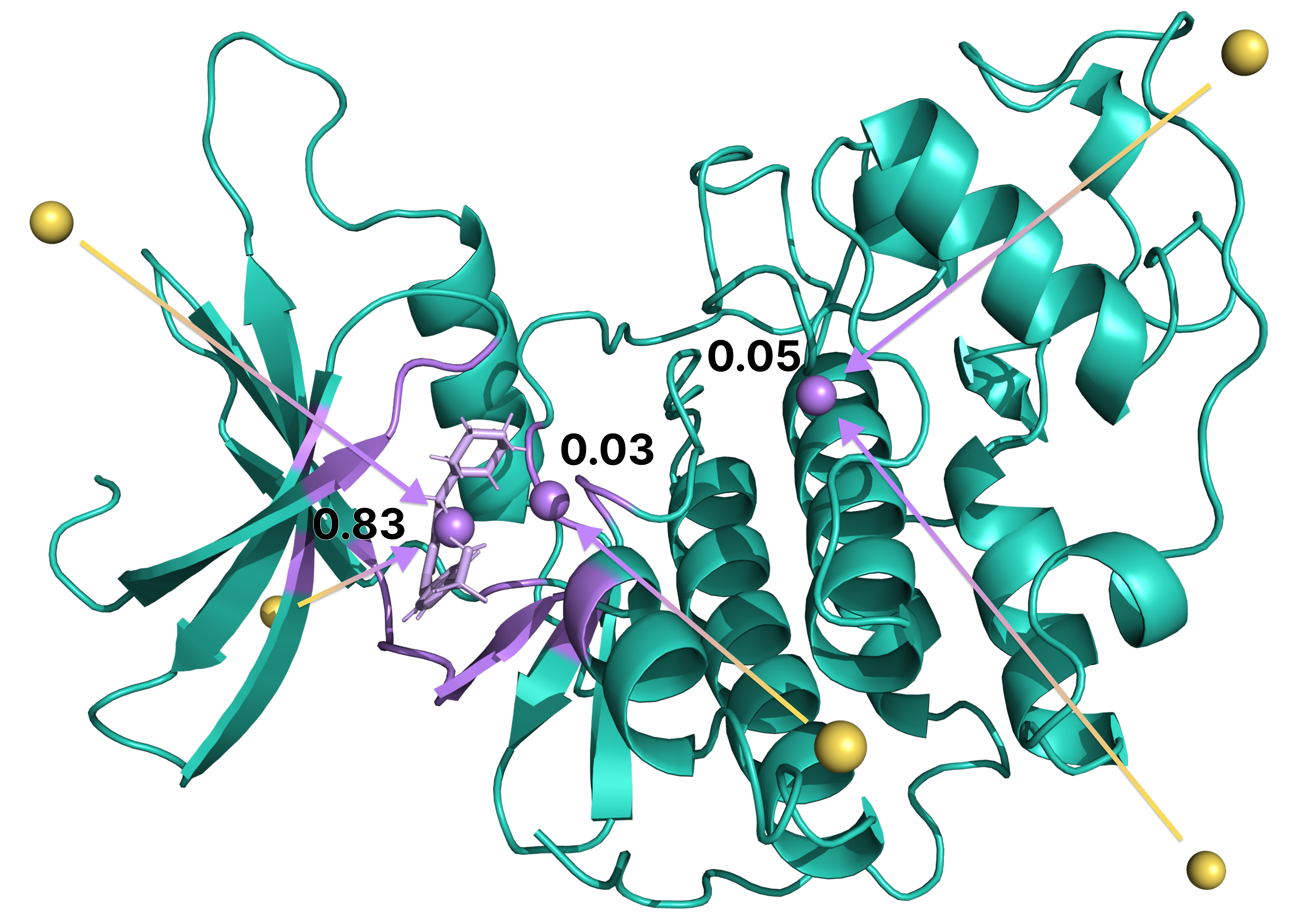}
    \hspace{15mm}
    \includegraphics[width=0.90\columnwidth]{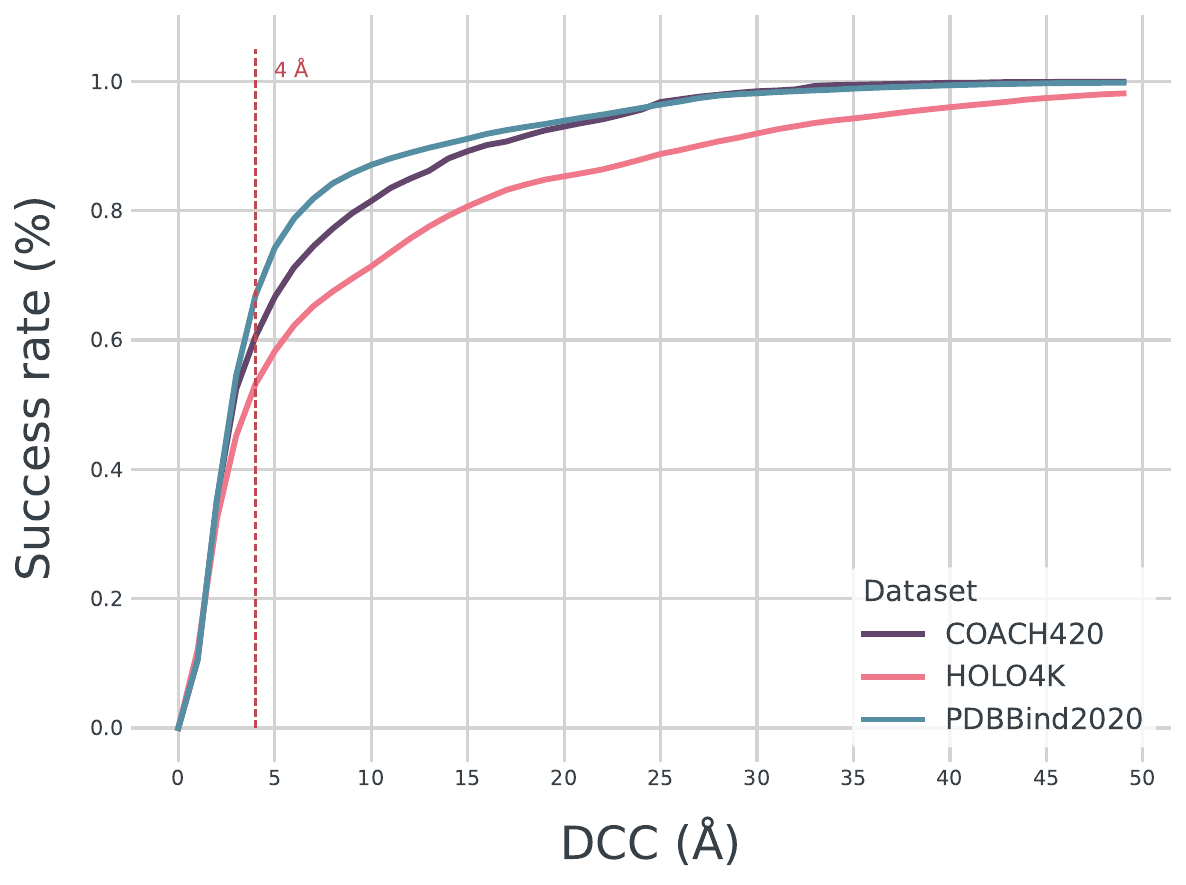}
    \caption{\textbf{Left:} Example of a prediction from our model: Initial positions of the virtual nodes are represented by the yellow spheres around the protein, the ground truth binding site is indicated by the light violet ligand, whereas violet regions on the protein represent the annotated binding site. The arrows indicate how the positions of the virtual nodes change from their initial positions. The violet spheres represent the clustered virtual node predictions with their associated self-confidence score. To simplify the visualization, not all initial positions of the virtual nodes are depicted in the figure.
    \textbf{Right:} DCC values for different thresholds. The x-axis
    denotes different thresholds for the distance of the predicted
    and known binding pocket center in \r{A}. Distances below 
    this threshold are considered as correctly found binding pockets.
    The y-axis denotes DCC success rate. 
    \label{fig:dcc_thresholds} \label{fig:example_1mxiA}}
\end{figure*}

\subsection{Results} 
Our experimental results demonstrate that our method VN-EGNN surpasses 
all prior approaches in terms of the DCC metric on COACH20, PDBbind2020
and even on the challenging HOLO4K dataset, see~\cref{tab:results}. 
On COACH20, VN-EGNN exhibits
the best DCA score and on PDBbind it yields the same DCA score as P2Rank.
Note that there is limited comparability with P2rank since this method uses
a different training set that might be closer to HOLO4K. HOLO4K contains many
complexes of symmetric proteins (see \cref{appsec:hyp}), which should be considered
as a strong domain shift to the training data and thus pose a problem for all methods
except P2rank.
For a more detailed discussion and a visual analysis we refer 
to \Cref{appsec:holo4k_artifacts}. Figure \ref{fig:example_1mxiA} shows
an example of our model predictions.

\subsection{Ablation study}
Our proposed method comprises three main components as compared
to typical other methods: 
(a) virtual nodes, 
(b) heterogenous message passing, and 
(c) pre-trained protein embeddings as node representations.
We ablate these three components in a set of experiments 
(see~\cref{tab:results_ablation}).
\emph{(a) Removing virtual nodes.} We compared our model to a standard EGNN framework to determine the added value of virtual nodes. An EGNN architecture, was also used in  \citet{zhang2023equipocket}, but significantly extension. In this study we compared how a standard EGNN performs to our method. \Cref{tab:results_ablation}, shows that the standard EGNN architecture did not perform well.
\emph{(b) Homogeneous message passing.} Our approach to message passing, which is applied in a sequential manner, was contrasted with the traditional method where updates across nodes occur in parallel or homogeneously. This evaluation was further enriched by employing identical MLPs for both graph and virtual nodes across all layers, providing a direct comparison of the impact of our message passing strategy.
\emph{(c) Atom type embeddings.}
We evaluated the impact of the type of embeddings, as outlined in \Cref{sec:egnn_bsi}. \Cref{tab:results_ablation} illustrates that, regardless of the initial embeddings used, our model surpasses all preceding approaches in achieving higher DCC success rate across the COACH420 and PDBbind2020 datasets, except P2Rank. This was accomplished by adopting a one-hot encoding scheme solely for the amino acid types, complemented by an additional category for the virtual nodes.

\section{Discussion and Conclusions.}
\textbf{Main findings.} We have introduced a 
novel method that extends EGNNs \citep{Satorras2021} with virtual 
nodes and a heterogeneous message-passing scheme. These new assets
improve the learning dynamics by ameliorating the oversquashing
problem and allow for learning representations of hidden geometric
entities. Concretely, we have developed this method for 
binding site identification, for which our experiments show that
VN-EGNN exhibits high predictive performance in terms of DCC 
and set a new state of the art on COACH420, HOLO4K and PDBbind2020.
We attribute our improvement largely to the direct prediction of binding site 
centers, rather than inferring them from the geometric center of segmented 
areas, a common practice in earlier methods. Relying on segmentation can lead to 
inaccuracies, especially if a single erroneous prediction impacts the calculated 
center. Overall, VN-EGNN yields highly accurate predictions of binding site centers.

\textbf{Comparison with previous work.} 
We again emphasize that we are proposing a method for binding site identification and not docking, 
both of which are fundamental, but distinct, tasks in 
computational chemistry. 
VN-EGNN is an equivariant graph neural network 
approach, similar to \citet{zhang2023equipocket}, 
however, with  several differences. 
Contrary to the majority of preceding methodologies, which predominantly 
utilized some form of surface representation \citep{1996Reduced, Eisenhaber1995}, 
our approach deviated by solely relying on residue 
level information. By only using residue-level information, i.e.
$\alpha$-carbons as physical nodes, our method is computationally much more efficient
both during training and inference because the input graphs are smaller. 
These results on using only alpha-carbons
underpins a finding by 
\citet{Jumper2021} that residue-level information implicitly 
contains all relevant side-chain conformations.

\textbf{Implications.} %
We expect that our empirical results and our 
new method re-new the interest in theoretically investigating the effect
of virtual nodes on the expressivity and the oversquashing problem of GNNs.
Practically, we envision that VN-EGNN will be a useful tool in molecular 
biology and structure-based drug design, which is regularly used to analyze 
proteins for potential binding pockets and their druggability. On the long
run, this could make the drug development process more time- and cost-efficient.

\textbf{Limitations.} Our method is currently limited to predicting
binding pockets of proteins similar to those in PDB. 
We expect that VN-EGNN
can also be applied to other physical or geometric problems with 
hidden geometric entities, such as particle flows, however, their performance
in these fields remains to be shown. Note that our method is not 
a docking method and thus cannot be used to dock ligands to protein structure.
However, our predicted binding sites can be used as proposal regions
for other methods, which could lead to improved 
performance and efficiency for docking methods. 

\textbf{Future work.} We aim at using VN-EGNN to annotate all proteins
in PDB with binding sites, and potentially a subset of the 200 million structures AlphaFold DB, 
with predicted binding pockets and release this annotated dataset.

\section*{Acknowledgements}
The ELLIS Unit Linz, the LIT AI Lab, the Institute for Machine Learning, are supported by the Federal State Upper Austria. We thank the projects AI-MOTION (LIT-2018-6-YOU-212), DeepFlood (LIT-2019-8-YOU-213), Medical Cognitive Computing Center (MC3), INCONTROL-RL (FFG-881064), PRIMAL (FFG-873979), S3AI (FFG-872172), DL for GranularFlow (FFG-871302), EPILEPSIA (FFG-892171), AIRI FG 9-N (FWF-36284, FWF-36235), AI4GreenHeatingGrids(FFG- 899943), INTEGRATE (FFG-892418), ELISE (H2020-ICT-2019-3 ID: 951847), Stars4Waters (HORIZON-CL6-2021-CLIMATE-01-01). We thank NXAI GmbH, Audi.JKU Deep Learning Center, TGW LOGISTICS GROUP GMBH, Silicon Austria Labs (SAL), FILL Gesellschaft mbH, Anyline GmbH, Google, ZF Friedrichshafen AG, Robert Bosch GmbH, UCB Biopharma SRL, Merck Healthcare KGaA, Verbund AG, GLS (Univ. Waterloo) Software Competence Center Hagenberg GmbH, T\"{U}V Austria, Frauscher Sensonic, TRUMPF and the NVIDIA Corporation. We acknowledge EuroHPC Joint Undertaking for awarding us access to Karolina at IT4Innovations, Czech Republic; MeluXina at LuxProvide, Luxembourg; LUMI at CSC, Finland.

\bibliography{newBib}
\bibliographystyle{icml2024}

\newpage

\appendix
\onecolumn 

\counterwithin{figure}{section}
\counterwithin{table}{section}
\counterwithin{equation}{section}
\renewcommand\thefigure{\thesection\arabic{figure}}
\renewcommand\thetable{\thesection\arabic{table}}
\renewcommand\theequation{\thesection.\arabic{equation}}

\thispagestyle{plain}

  {\center\baselineskip 18pt
                       \toptitlebar{\Large\bf Appendix: \\
                       VN-EGNN: E(3)-Equivariant Graph Neural Networks with Virtual Nodes Enhance Protein Binding Site Identification}\bottomtitlebar}

This appendix is structured as follows:
\begin{itemize}
\item We provide an overview table of the \textbf{notation} in~\cref{sec:notation}.
\item We then provide a more formal \textbf{problem statement} in~\cref{problem}.
\item For self-containedness, we give some \textbf{background on group theory}
and equivariance in~\cref{eqBackground}.
\item The proof of the \textbf{equivariance of VN-EGNNs} is given in~\cref{appsec:proof}.
\item Further details on the \textbf{experimental settings} are given in~\cref{appsec:expGeneral}. %
\item Details on the \textbf{initial experiment} can be found in~\cref{appsec:intial}.
\item An investigation of the approximate invariance with 
respect to initialization of the virtual node coordinates 
is proved in~\ref{sec:initial_coords}.
\item We included some \textbf{visualizations} in~\cref{appsec:visualization}.
\item Statistics on the differences between the datasets and potential \textbf{domain shifts and HOLO4K} are provided in~\cref{appsec:holo4k_artifacts}.
\item Finally, we included additional theory and experimental results on the \textbf{expressivity of VN-EGNN} in~\cref{appsec:expressivity}.
\end{itemize}

\clearpage

\section{Notation Overview}
\label{notation}\label{sec:notation}

\vspace{0.5cm}

\begin{table}[H]
\centering
\begin{threeparttable}

\begin{tabular}{l c l}
\toprule
Definition & Symbol/Notation & Type  \\ 
\midrule
number of physical nodes & $N$ & $\mathbb{N}$ \\
number of virtual nodes & $K$ & $\mathbb{N}_0$ \\
number of known binding pockets & $M$ & $\mathbb{N}_0$ \\
dimension of node features & $D$ & $\mathbb{N}$ \\
dimension of messages & $E$ & $\mathbb{N}$ \\
number of message passing layers/steps & $L$ & $\mathbb{N}$  \\
\midrule
node indices & $i,j,k,n$ & $\{1,...,K\}$ or $\{1,...,N\}$ \\
binding pocket index & $m$ & $\{1,...,M\}$ \\
layer/step index & $l$ & $\{1,...,L\}$ \\
index set of 10 nearest neighbor atoms & $\mathcal{N}(i)$ & $\{1,...,N\}^{10}$ \\
\midrule
physical node coordinates & $\Rx_i^l$ & $\mathbb{R}^3$ \\
virtual node coordinates & $\Rz_j^l$ & $\mathbb{R}^3$  \\
physical node feature representation & $\Bh_i^l$ & $\mathbb{R}^D$ \\
virtual node feature representation & $\Bv_j^l$ & $\mathbb{R}^D$ \\
edge feature between physical nodes & $a_{ij}$ & $\mathbb{R}$ \\
edge feature between physical and virtual node & $d_{ij}$ & $\mathbb{R}$ \\
\midrule
matrix of all physical node coordinates & $\RX$ & $\mathbb{R}^{3 \times N}$ \\
matrix of all virtual node coordinates & $\RZ$ & $\mathbb{R}^{3 \times K}$ \\
matrix of all physical node feature representations & $\BH$ & $\mathbb{R}^{D \times N}$ \\
matrix of all virtual node feature representations & $\BV$ & $\mathbb{R}^{D \times K}$ \\
\midrule
ground-truth node label & $y_n$ & \{0,1\} \\
predicted node label & $\hat y_n$ & $\left[0,1\right]$ \\
ground-truth binding site center & $\Ry_m$ & $\mathbb{R}^3$ \\
prediction of binding site center & $\hat\Ry_k$ & $\mathbb{R}^3$ \\
\midrule
messages* & $\Bm_{ij}^{(aa)} \!, \Bm_{ij}^{(av)} \!, \Bm_{ij}^{(va)}$ & $\mathbb{R}^E$\\
neural networks for message passing*: & & \\
$\quad$ message calculation & $\Bphi_{e^{(aa)}}, \Bphi_{e^{(av)}}, \Bphi_{e^{(va)}}$ & $ \mathbb{R}^D \! \times \! \mathbb{R}^D \! \times \! \mathbb{R} \! \times \! \mathbb{R} \rightarrow \mathbb{R}^E$ \\
$\quad$ coordinate update & $\phi_{\Rx^{(aa)}}, \phi_{\Rx^{(av)}}, \phi_{\Rx^{(va)}}$ \tnote{a}   & $\mathbb{R}^E \rightarrow \mathbb{R}$ \\
$\quad$ feature update & $\Bphi_{h^{(aa)}}, \Bphi_{h^{(av)}}, \Bphi_{h^{(va)}}$ & $\mathbb{R}^D \! \times \! \mathbb{R}^E \rightarrow \mathbb{R}^D$ \\
\midrule
segmentation loss & $\mathcal L_{\mathrm{segm}}$ & $\mathbb{R}^N \times \mathbb{R}^N \rightarrow \mathbb{R}$ \\
binding site center loss & $\mathcal L_{\mathrm{bsc}}$ & $\mathbb{R}^{3 \times M} \times \mathbb{R}^{3 \times K} \rightarrow \mathbb{R}$ \\

\bottomrule
 \end{tabular}
  \begin{tablenotes}
    \item \scriptsize{* The superscripts (aa), (av) and (va) represent the message passing direction (\textbf{atom} to \textbf{atom}, \textbf{atom} to \textbf{virtual node}, \textbf{virtual node} to \\
    \textbf{atom}).}
 \end{tablenotes}

\caption{Overview of used symbols and notations}
\label{tab:notation}

\end{threeparttable}

\end{table}

\clearpage

\section{Problem Statement}
\label{problem}

\subsection{Representation of proteins.}
\label{protein}

The 3D structure of a protein is usually given by some measurement 
of its atoms that form the primary amino acid sequence of the protein and the absolute coordinates for the atoms are given as 3D points $\Rx \in \mathbb{R}^3$.
The atoms themselves as well as the amino acids are characterized by 
their physical, chemical and biological properties. We assume that these
properties are summarized by feature vectors $\Bh \in \mathbb{R}^D$, which 
are located at the atom centers (either of all the atoms or only the ones
forming the protein backbone). We formally represent proteins by a neighborhood graph $\mathcal P=\left(\mathcal P_N, \mathcal P_E\right)$ with $N$ atom-property pairs, 
i.e. $\mathcal P_N=\{(\Rx_n,\Bh_n)\}_{n=1}^N$ with $\Rx_n \in \mathbb{R}^3$ and $\Bh_n \in \mathbb{R}^D$ and a set of directed 
edges $\mathcal P_E$ which consist of atom-property pairs $\left(i,j\right)$. Each node $i$ has incoming edges from
the 10 nearest nodes $j$ that are closer 
than 10\AA\ according 
to the Euclidean distance $\lVert \Rx_i - \Rx_j \rVert$.

\subsection{Representation of binding sites.}
\label{bindingSite}
Binding sites are regions around or within proteins, to which ligands can potentially bind. 
Basically, one can either describe binding sites \textit{explicitly} or \textit{implicitly}. 
In their explicit representation binding sites would be directly described 
by the location specifics of the regions, where ligands are located, 
especially by a region center point. In their implicit representation, 
binding sites would be described by the atoms of the protein, which 
surround the ligand. Atoms close to the ligand would be marked as 
binding site atoms. It might be worth mentioning, that several 
binding sites per protein are possible.

Formally, for the explicit representation, 
we describe the (experimentally observed) binding site center
points of $M$ distinct binding sites by $\Ry_m \in \mathbb{R}^3$
with $1\le m \le M$. For the implicit representation, we assign to 
each protein atom $n$ a label $y_n \in \{0, 1\}$, which is set to
$1$ if the atom center is within the threshold distance of 
observed binding ligands, and $0$ otherwise.

\subsection{Objective.}

From an abstract point of view,
we want to learn a predictive machine learning model 
$\Fcal$, parameterized by $\omega$, which maps proteins characterized by 
the positions of their atoms together with their properties to a binary 
prediction per atom, whether it might form a binding site and to $K$ 3D 
coordinates representing binding site region center points:

\begin{align} 
\text{\normalsize $\Fcal_\omega$}: \bigtimes_{n=1}^N \left( \underbrace{\mathbb{R}^3 \times \mathbb{R}^D}_{\substack{\text{protein 3D atom} \\ \text{coords with} \\ \text{$D$-dim features}}} \right) & \mapsto \underbrace{[0,1]^N}_{\substack{\text{sem. segm.} \\ \text{protein atoms} \\  \\ \\ \text{ \normalsize $\coloneqq \Fcal_\omega^{\mathrm{segm}}$}}} \times \underbrace{ \bigtimes_{k=1}^K \mathbb{R}^3}_{\substack{\text{pos. pred.} \\ \text{virt. nodes} \\ \\ \text{ \normalsize $\coloneqq \Fcal_\omega^{\mathrm{bsc}}$}}}\\
\text{\normalsize $\Fcal_\omega$}\left((\Rx_1, \Bh_1),\ldots, (\Rx_N, \Bh_N)\right)&=\left( (\hat y_1, \ldots, \hat y_N), (\hat \Ry_1, \ldots, \hat \Ry_K) \right) \nonumber \\
\text{\normalsize $\Fcal_\omega^{\mathrm{segm}}$}\left((\Rx_1, \Bh_1),\ldots, (\Rx_N, \Bh_N)\right)&\coloneqq\mathrm{proj}_1 \text{\normalsize $\Fcal_\omega$}\left((\Rx_1, \Bh_1),\ldots, (\Rx_N, \Bh_N)\right) \nonumber \\
\text{\normalsize $\Fcal_\omega^{\mathrm{bsc}}$}\left((\Rx_1, \Bh_1),\ldots, (\Rx_N, \Bh_N)\right)&\coloneqq\mathrm{proj}_2 \text{\normalsize $\Fcal_\omega$}\left((\Rx_1, \Bh_1),\ldots, (\Rx_N, \Bh_N)\right) \nonumber,
\end{align}
where $\mathrm{proj}_i$ is a projection, that gives the $i$-th component (i.e., prediction of the semantic segmenation part or coordinate predicitons or virtual nodes). Note, that for our predictive model, we use a fixed number $K$ of binding point centers, while indeed $M$ might have been observed for a specific protein.

\subsection{Utilized Loss Functions.}
\label{lossFnct}

\textbf{Segmentation loss.} For semantic segmentation (i.e., the prediction of $\Fcal_\omega^{\text{segm}}$), we use a Dice loss, that is based on the continuous Dice coefficient \citep{Shamir2018}, with $\epsilon=1$:
\begin{align}
\mathcal L_{\mathrm{dice}}=\mathrm{Dice} \left( (y_1, \ldots, y_N) , (\hat y_1, \ldots, \hat y_N) \right) \coloneqq 1 - \frac{2 \ \sum_{n=1}^N y_n \ \hat y_n + \epsilon}{\sum_{n=1}^N y_n + \sum_{n=1}^N \hat y_n + \epsilon}
\end{align}
Perfect predictions lead to a Dice loss of $0$, while perfectly wrong predictions would lead to a Dice of $1$ (in case $\epsilon=0$ and the denominator is $>0$).

\textbf{Binding site center loss.} For prediction of the binding site region center points (i.e., the prediction of $\Fcal_\omega^{\text{bsc}}$), we use the (squared) Euclidean distance between the set of predicted points and the set of observed ones. More specifically, we assume to be given $M$ observed center points $\{\Ry_1,\ldots,\Ry_M\}$. Each of the binding site center points should be detected by at least one of the $K$ outputs from $\Fcal_\omega^{\text{bsc}}$, which translates to using the minimum squared distance to any predicted center point for any of the observed center points:

\begin{align}
\mathcal L_{\mathrm{bsc}}=\mathrm{Dist} \left( \{\Ry_1, \ldots, \Ry_M\} , \{\hat \Ry_1, \ldots, \hat \Ry_K\} \right) \coloneqq \frac{1}{M} \sum_{m=1}^M
\min_{k \in {1,\ldots,K}} \lVert \Ry_m - \hat \Ry_k \rVert^2.
\end{align}

Our optimization objective is then the sum of the $\mathrm{Dice}$ and the $\mathrm{Dist}$ loss: 
\begin{align}
\alpha \mathrm{Dice} \left( (y_1, \ldots, y_N) , (\hat y_1, \ldots, \hat y_N) \right)+\mathrm{Dist} \left( \{\Ry_1, \ldots, \Ry_M\} , \{\hat \Ry_1, \ldots, \hat \Ry_K\} \right)
\end{align}
with the hyperparameter $\alpha=1$.

\clearpage

\section{Background on Group Theory and Equivariance}
\label{eqBackground}

A group in the mathematical sense is a set $G$ along with a binary operation $\circ: G \times G \rightarrow G$ with the following properties:

\begin{itemize}
\item \textit{Associativity}: The group operation is associative, i.e. $(g \circ h) \circ k = g \circ (h \circ k)$ for all $g, h, k \in G$.
\item \textit{Identity:} There exists a unique identity element $e \in G$, such that $e \circ g = g \circ e = g$ for all $g \in G$.
\item \textit{Inverse:} For each $g \in G$ there is a unique inverse element $g^{-1} \in G$, such that $g \circ g^{-1} = g^{-1}  \circ g= e$. 
\item \textit{Closure:} For each $g, h \in G$ their combination $g \circ h$ is also an element of $G$.
\end{itemize}

A group action of group $G$ on a set $X$ is defined as a set of mappings $T_g: X \rightarrow X$ which associate each element $g \in G$ with a transformation on $X$, whereby the identity element $e \in G$ leaves $X$ unchanged ($T_e(x) = x \quad \forall x \in X$). 

An example is the group of translations $\mathbb{T}$ on $\mathbb{R}^n$ with group action $T_t(x) = \Rx+\Rt \quad \forall \Rx,\Rt \in \mathbb{R}^n$, which shifts points in $\mathbb{R}^n$ by a vector $\Rt$.

A function $f: X \rightarrow Y$ is equivariant to group $G$ with group action $T_g$ if there exists an equivalent group action $S_g: Y \rightarrow Y$ on $G$ such that $$f(T_g(x)) = S_g(f(x)) \quad \forall x \in X, g \in G.$$ 

For example, a function $f: \mathbb{R}^n \rightarrow \mathbb{R}^n$ is translation-equivariant if a translation of an input vector $\Rx \in  \mathbb{R}^n$ by $\Rt \in  \mathbb{R}^n$ leads to the same transformation of the output vector $f(x) \in \mathbb{R}^n$, i.e. $f(\Rx+\Rt) = f(\Rx) +\Rt$.

Equivariant graph neural networks (EGNNs) $\psi$ as defined by \citet{Satorras2021} exhibit three types of equivariances:
\begin{enumerate}
\item \textit{Translation equivariance:} EGNNs are equivariant to column-wise addition of a vector $\Rt \in \mathbb{R}^n$ to all points in a point cloud $\RX \in \mathbb{R}^{n \times N}$: $\psi(\RX+\Rt) = \psi(\RX)+\Rt$.

\item \textit{Rotation and reflection equivariance}: Rotation or reflection of all points in the point cloud by multiplication with an orthogonal matrix $\RR \in \mathbb{R}^{n \times n}$ leads to an equivalent rotation of the output coordinates: $\psi(\RR \RX) = \RR \psi(\RX)$.

The group spanning all translations, rotations and reflections in $\mathbb{R}^n$ is called Euclidean group, denoted E($n$), as it preserves Euclidean distances. A proof for E(n)-equivariance of VN-EGNN can be found in  \Cref{appsec:proof}.

\item \textit{Permutation equivariance:} The numbering of elements in a point cloud or graph nodes does not influence the output, i.e. multiplication with a permutation matrix $\BP \in \mathbb{R}^{N \times N}$ leads to the same permutation of output nodes: $\psi(\RX\BP) = \psi(\RX)\BP$. This property holds for message passing graph neural networks in general, as they aggregate and update node information based on local neighborhood structure, regardless of the order in which nodes are presented.
\end{enumerate}

\clearpage

\section{Equivariance of VN-EGNN}
\label{appsec:proof}
In this section we show that the equivariance property of EGNN \citep{Satorras2021} extends to VN-EGNN, i.e., that rotation and reflection by an orthogonal matrix $\RR \in \mathbb{R}^{3 \times 3}$, and translation by a vector $\Rt \in \mathbb{R}^3$ of atom and virtual node coordinates leads to an equivalent transformation of output coordinates while leaving node features invariant when applying the message passing steps of VN-EGNN.

\setcounter{proposition}{0}
\begin{proposition} \label{eqpropFormal} (more formal)
    E($3$) equivariant graph neural networks  with virtual nodes ($\mathrm{VN\text{-}EGNN}$) as defined by the message passing scheme 
    $\left( \RX^{l+1},\BH^{l+1},\RZ^{l+1}, \BV^{l+1} \right) = \mathrm{VN\text{-}EGNN}  \left(\RX^{l},\BH^{l},\RZ^{l}, \BV^{l},\BA \right)$ in 
    \crefrange{eq:message1}{eq:update3} are equivariant with respect 
    to reflections and roto-translations of the input and virtual node coordinates, i.e., the following holds (equivariance to reflections and roto-translations):
    \begin{align}
       \left(\RR\RX^{l+1}+\Rt, \BH^{l+1}, \RR \RZ^{l+1} + \Rt, \BV^{l+1}\right) = \mathrm{VN\text{-}EGNN}\left(\RR\RX^l+\Rt, \BH^l, \RR\RZ^l+\Rt,  \BV^l\right), \label{eq:inv_equiv}
    \end{align}
    where the addition $\RX^l+\Rt$ is defined as column-wise addition of
    the vector $\Rt$ to the matrix $\RX$.
\end{proposition}

\begin{proof}
We use the notation from \Cref{subsec:notation} and proceed by tracking the propagation of node roto-translations through the VN-EGNN network. First, we want to show invariance in \cref{eq:message1} in phase I of message passing, equivalently to \citet{Satorras2021}, i.e.:
\begin{align}
   \Bm_{ij}^{(aa)} = \Bphi_{e^{(aa)}}(\Bh_i^l,\Bh_j^l,\lVert \RR\Rx_i^l + \Rt - \lbrack \RR\Rx_j^l + \Rt\rbrack \rVert^2, a_{ij}) = \Bphi_{e^{(aa)}}(\Bh_i^l,\Bh_j^l,\lVert \Rx_i^l - \Rx_j^l \rVert^2, a_{ij})
\end{align}
Assuming the initial node features $\Bh^0_i$ and edge representations $a_{ij}$ do not encode information about the original coordinates $\Rx^0_i$, it remains to be shown that the Euclidean distance between two nodes is also invariant to translation and rotation:
\begin{align}
    \begin{split}
       \lVert \RR\Rx_i^l + \Rt - \lbrack \RR\Rx_j^l + \Rt\rbrack \rVert^2 &= \lVert \RR\Rx_i^l - \RR\Rx_j^l \rVert^2 \\
       &= (\Rx_i^l - \Rx_j^l)^\top \RR^\top \RR (\Rx_i^l - \Rx_j^l) \\
       &= (\Rx_i^l - \Rx_j^l)^\top \mathbf{I} (\Rx_i^l - \Rx_j^l) \\
       &= \lVert \Rx_i^l - \Rx_j^l \rVert^2  \label{distance_inv}
    \end{split}
\end{align}
Consequently, the sum over messages (\cref{eq:sum1}) and the feature update function (\cref{eq:update1}), which only uses the summed messages and previous node features as input, are invariant as well, leaving the intermediate output feature representations $\Bh^{l+\sfrac{1}{2}}_i$ independent of coordinate transformations.

For the remaining equation (\cref{eq:x_coord1}) of phase I the equivariance property can be shown as follows, where \cref{distance_inv} is used in the first equality:

\begin{align}
    \begin{split}
        \RR\Rx_i^l+\Rt + \frac{1}{\mathcal N(i)} \sum_{j\in \mathcal N(i)} &\frac{\RR\Rx_i^l+\Rt - \lbrack \RR\Rx_j^l+\Rt \rbrack}{\lVert \RR\Rx_i^l+\Rt - \lbrack \RR\Rx_j^l+\Rt \rbrack \rVert} \phi_{x^{aa}}(\Bm_{ij}^{(aa)}) \\
        &= \RR\Rx_i^l+\Rt + \frac{1}{\mathcal N(i)} \sum_{j\in \mathcal N(i)} \frac{\RR\Rx_i^l+\Rt - \lbrack \RR\Rx_j^l+\Rt \rbrack}{\lVert \Rx_i^l - \Rx_j^l \rVert} \phi_{x^{aa}}(\Bm_{ij}^{(aa)}) \\
        &= \RR\Rx_i^l+\Rt + \frac{1}{|\mathcal N(i)|} \;\; \RR\!\! \sum_{j\in \mathcal N(i), j \neq i} \frac{\Rx_i^l - \Rx_j^l}{\lVert \Rx_i^l - \Rx_j^l \rVert} \phi_{x^{aa}}(\Bm_{ij}^{(aa)}) \\
        &= \RR \left(\Rx_i^l + \frac{1}{|\mathcal N(i)|} \sum_{j\in \mathcal N(i), j \neq i} \frac{\Rx_i^l - \Rx_j^l}{\lVert \Rx_i^l - \Rx_j^l \rVert}\phi_{x^{aa}}(\Bm_{ij}^{(aa)})\right) + \Rt\\
        &= \RR\Rx_i^{l+\sfrac{1}{2}}+\Rt 
        \label{eq:equivariance}
    \end{split}
\end{align}

In phase II of message passing, we input the updated physical node coordinates $\RR \Rx_i^{l+\sfrac{1}{2}} + \Rt$ from \cref{eq:equivariance} together with virtual node coordinates $\RR \Rz_j^l + \Rt$, both subjected to identical rotation and translation. Invariance of 
\cref{eq:message2,eq:sum2,eq:update2} can be deduced similarly to above, using the invariance properties of node features $\Bh^{l+\sfrac{1}{2}}_i$ and $\Bv^l_i$, edge features $a_{ij}$ and $d_{ij}$, and Euclidean distance (\cref{distance_inv}): 

\begin{align}
    \begin{split}
        \Bm_{ij}^{(av)} &= \Bphi_{e^{(av)}}(\Bh_i^{l+\sfrac{1}{2}}, \Bv_j^l, \lVert \RR\Rx_i^{l+\sfrac{1}{2}}+\Rt - \lbrack \RR\Rz_j^{l}+\Rt \rbrack \rVert^2, d_{ij}) \\
        &= \Bphi_{e^{(av)}}(\Bh_i^{l+\sfrac{1}{2}}, \Bv_j^l, \lVert \Rx_i^{l+\sfrac{1}{2}} - \Rz_j^l \rVert^2, d_{ij}) 
    \end{split}
\end{align}

Thus, the output virtual node features $\Bv_j^{l+1}$ are invariant to roto-translations of node coordinates. Note
that reflections are also covered by~\cref{eq:equivariance}
since the distance of two points does not change under
reflection.

Equivariance of output virtual node coordinates $\Rz_j^{l+1}$ follows analogously to \cref{eq:equivariance}:
\begin{align}
    \RR\Rz_j^l + \Rt + \frac{1}{N} \sum_{i=1}^N \frac{\RR\Rx_i^{l+\sfrac{1}{2}}+\Rt - \lbrack \RR\Rz_j^l +\Rt \rbrack}{\lVert \RR\Rx_i^{l+\sfrac{1}{2}}+\Rt - \lbrack \RR\Rz_j^l +\Rt \rbrack \rVert}\phi_{x^{av}}(\Bm_{ij}^{(av)})=\RR\Rz_j^{l+1}+\Rt
\end{align}

The same derivations of message invariance 

\begin{align}
    \begin{split}
        \Bm_{ij}^{(va)} &= \Bphi_{e^{(va)}}(\Bv_i^{l+1}, \Bh_j^{l+\sfrac{1}{2}},\lVert R\Rz_i^{l+1}+\Rt - \lbrack \RR\Rx_j^{l+\sfrac{1}{2}}+\Rt \rbrack \rVert^2, d_{ji}) \\
        &= \Bphi_{e^{(va)}}(\Bv_i^{l+1}, \Bh_j^{l+\sfrac{1}{2}}, \lVert \Rz_i^{l+1} - \Rx_j^{l+\sfrac{1}{2}} \rVert^2, d_{ji})
    \end{split}
\end{align}

 and coordinate equivariance 
 \begin{align}
     \RR\Rx_j^{l+\sfrac{1}{2}} + \Rt + \frac{1}{K} \sum_{i=1}^K \frac{\RR \Rz_i^{l+1} + \Rt - \lbrack \RR \Rx_j^{l+\sfrac{1}{2}} +\Rt \rbrack}{\lVert \RR \Rz_i^{l+1} + \Rt - \lbrack \RR \Rx_j^{l+\sfrac{1}{2}} +\Rt \rbrack \rVert} \phi_{x^{va}}(\Bm_{ij}^{(va)}) = \RR\Rx_j^{l+1} + \Rt 
 \end{align}
 
 can be applied to phase III (\crefrange{eq:message3}{eq:update3}), proving that invariance of feature representations $\Bh_j^{l+1}$ and equivariance of coordinates $\Rx_j^{l+1}$ holds true for physical nodes as well, thus, proving \cref{eqpropFormal}.

\end{proof}

\clearpage

\section{Experimental Settings}
\label{appsec:expGeneral}

\subsection{Datasets}
\label{appsec:experimental_settings}

\textbf{scPDB \citep{desaphy2015sc}} is a frequently utilized dataset for binding site prediction \citep{kandel2021puresnet, stepniewska2020improving}, encompassing both protein and ligand structures. We employed the 2017 release of scPDB in the training and validation. This release comprises 17,594 structures, 16,034 entries, 4,782 proteins, and 6,326 ligands. Structures were clustered based on their Uniprot IDs. From each cluster, protein structures with the longest sequences were selected, in alignment to the strategies used in \citet{kandel2021puresnet} and \citet{zhang2023equipocket}. 

(Source: \url{https://github.com/jivankandel/PUResNet/blob/main/scpdb_subset.zip})

\textbf{PDBbind \citep{2004The}} is a widely recognized dataset integral to the study of protein-ligand interactions. This dataset provides detailed 3D structural information of proteins, ligands, and their respective binding sites, complemented by rigorously determined binding affinity values derived from laboratory evaluations. For our work, we draw upon the v2020 edition, which is divided into two sets: the general set (comprising 14,127 complexes) and the refined set (containing 5,316 complexes). While the general set encompasses all protein-ligand interactions, only the refined set, curated for its superior quality from the general collection, is used in our experiments. 

(Source: \url{http://www.pdbbind.org.cn/download/PDBbind_v2020_refined.tar.gz})

\textbf{COACH420 \textnormal{and} HOLO4K} are benchmark datasets utilized for the prediction of binding sites, as originally detailed by~\citet{krivak2018p2rank}. Following the methodologies of~\citet{krivak2018p2rank, Mylonas2021, aggarwal2021deeppocket}, we adopt the so-called {\texttt{mlig}} subsets from each of these datasets, which encompass the significant ligands pertinent to binding site prediction. Note that the HOLO4K contains many multi-chain systems and complexes with multiple copies of the protein (see Section~\cref{appsec:holo4k_artifacts}),
such that this dataset's distribution is 
strongly differs from the other datasets. 

(Source: \url{https://github.com/rdk/p2rank-datasets})

For comprehensive data preparation across all datasets, solvent atoms were excluded and erroneous structures were removed.

\subsection{Hyperparameters and hyperparameter selection}
\label{appsec:hyp}

\Cref{tab:hyperparams} shows the evaluated hyperparameters. Bold indicates the parameters used in final model.

\begin{table}[!ht]
    \centering
    \begin{tabular}{lc}
         hyperparmater & considered and \textbf{selected} values \\ \midrule
         optimizer & \{\textbf{AdamW}, Adam  \} \\
         learning rate & $\{ \mathbf{0.001},0.0001 \}$  \\
         activation function & \{  \textbf{SiLU}, ReLU \} \\
         dimension of node features $D$ & $\{ 20, 30, \textbf{100} \}$ \\
         dimension of the messages $P$ & $\{ 40, 50, \textbf{100} \}$ \\
         number of message passing layers/steps $L$ & $\{ 2,3,4,\textbf{5}\}$ \\
        number of virtual nodes $K$ & $\{4,\textbf{8}\}$ \\
        Huber loss $\delta$ & $\{1\}$
    \end{tabular}
    \caption{A list of considered and selected hyperparameters.}
    \label{tab:hyperparams}
\end{table}

\subsection{Inference Time Evaluation}

Our model's prediction time depends on the protein size. We demonstrate this by reporting inference times for two proteins:
\begin{itemize}
    \item 3LPK protein (910 residues), 2.367 seconds
    \item 1ODI protein (1,410 residues), 3.809 seconds
\end{itemize}

\subsection{Fibonacci grid}
\label{appsec:fibonacci}
The Fibonacci grid \citep{Swinbank2006} offers a 
solution for evenly distributing points on a sphere. 
We chose this method to obtain the starting coordinates of virtual 
nodes for its simplicity and efficiency. To prevent virtual nodes from starting at identical locations in subsequent iterations, we apply random rotation of the sphere for each sample in every epoch.

\clearpage

\section{Initial experiment}
\label{appsec:intial}

\Cref{fig:dcc_dcc_gb} shows the training curves 
for a VN-EGNN during the development phase. VN-EGNN were
only trained to minimize 
the segmentation loss $\mathcal L_{\mathrm{segm}}$. 
Even in the absence of a the binding site 
center loss $\mathcal L_{\mathrm{bsc}}$, the virtual nodes 
tend to converge towards the actual binding site center. 
This finding inspired us to further refine the position of 
the virtual nodes by including it directly to the optimization objective, which further improved the results significantly.

\begin{figure}[ht]
    \centering
    \includegraphics[width=\textwidth]{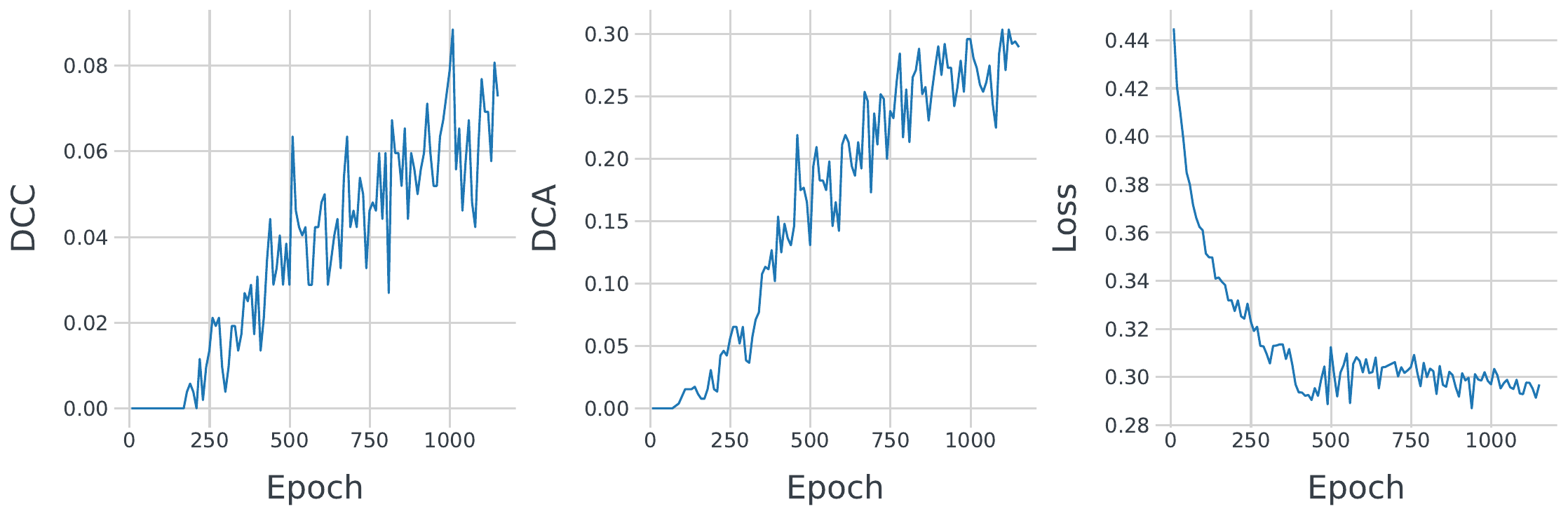}
    \caption{Validation curves of a VN-EGNN during development. Despite only being trained to minimize the segmentation loss, the virtual nodes converged towards the known binding sites. \textbf{Left}: DCC success rate during training. \textbf{Middle}: DCA success rate during training. \textbf{Right}: Segmentation loss during training.}
    \label{fig:dcc_dcc_gb}
\end{figure}

\clearpage

\section{Invariance and initial coordinates of virtual nodes.}
\label{sec:initial_coords}

As stated in Section~\ref{sec:properties_vnegnn}, we randomly rotate the grid of initial virtual nodes for
each sample during training in order to achieve
approximate invariance with respect to the initial 
coordinates of the virtual nodes. In an additional 
experiment, we evaluate how different relative positions
of the virtual nodes to the protein affect the method performance.
To investigate this, we randomly 
rotate the protein within the Fibonacci grid and perform
inference on the proteins of the benchmarking datasets.
The results of this experiment are presented in 
Table~\ref{tab:varRot}. Overall,
the results show that there is hardly any change of 
the DCC and DCA metric across different random rotations. 
Therefore we conclude that our method VN-EGNN achieves approximate
invariance to the initial coordinates of the virtual nodes
via data augmentation during training.

\vspace{1cm}

\begin{table}[ht]
\centering

\begin{tabular}{l c c }
\toprule
Dataset & DCC & DCA  \\  \midrule
COACH420 & 0.612(0.005) & 0.741(0.006)  \\
HOLO4Kd & 0.524(0.002) & 0.632(0.002)  \\
PDBbind2020 & 0.702(0.001) & 0.833(0.002) \\

\bottomrule
\end{tabular}

\caption{ Mean Performance at binding site identification in 
terms of DCC and DCA success rates together with their standard deviations (in parentheses). Means and standard deviations are across different random rotations of the Fibonacci grid.}
\label{tab:varRot}

\end{table}

}

\clearpage
\section{Visualizations}
\label{appsec:visualization}
\Cref{fig:protein_visualizations} shows exemplary model predictions visualized with Pymol. 

\begin{figure}[ht]
    \centering
    \includegraphics[width=1.0\textwidth]{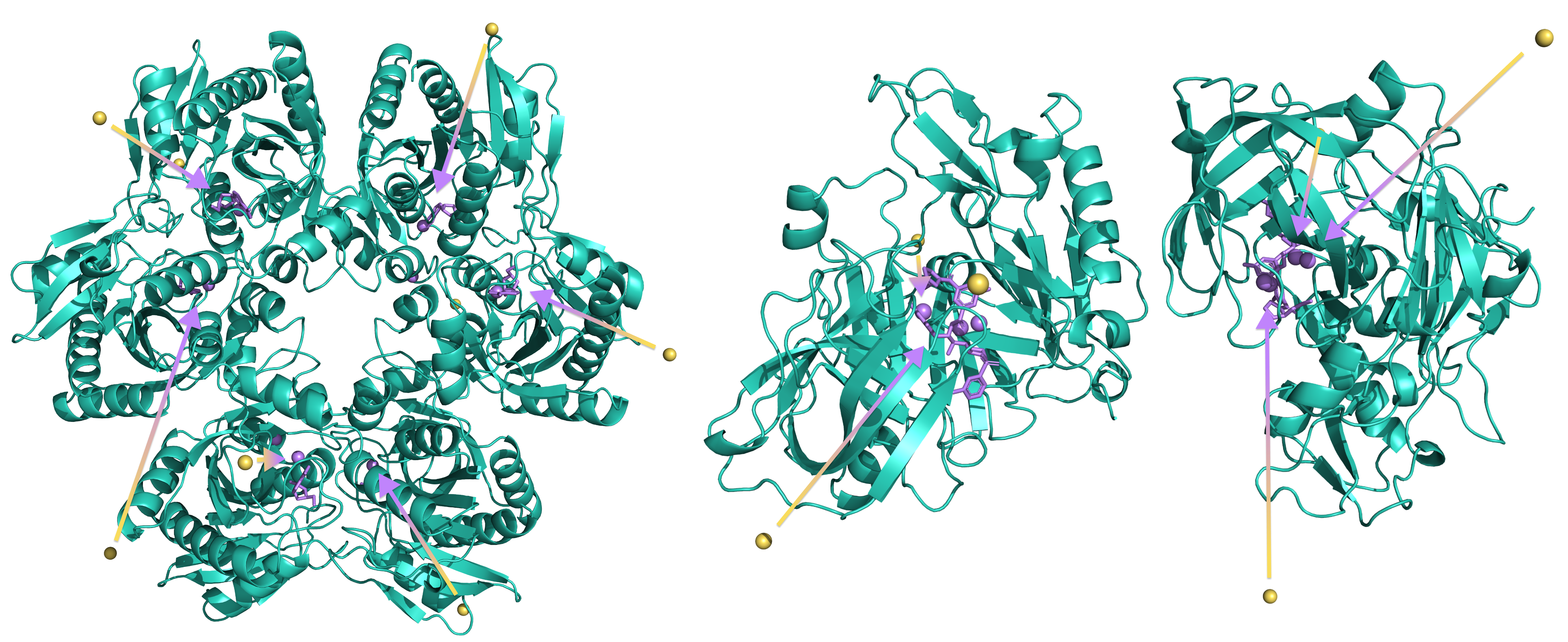}
    \caption{Examples of Detected Binding Sites: Visualization and Analysis.
We visualized two distinct proteins using Pymol, where the initial positions of the virtual nodes are represented by yellow spheres, were the violet spheres indicate the virtual nodes following $L$ message passing steps. The violet molecules indicate the position of the ligand as in the original PDB file. The arrows indicate the starting positions and the predicted positions of the virtual nodes. The visualization demonstrates that our model distributes the virtual nodes amongst various possible binding positions. The visualizations show the predicted positions after applying clustering as described in Section \ref{subsec:impl_details}. To simplify the visualization, not all initial positions of the virtual nodes are depicted in the figure.
    \textbf{Left}:1odi \textbf{Right}: 3lpk.
    \label{fig:protein_visualizations}}
\end{figure}

Further, we visualize the learned virtual node features, grouped by the corresponding protein’s target classification according to the ChEMBL \citep{gaulton2011chembl} database to analyze whether these representations contain relevant information about the protein/binding pocket (\Cref{fig:tsne_plot}).

\begin{figure}[ht]
    \centering
    \includegraphics[width=0.7\textwidth]{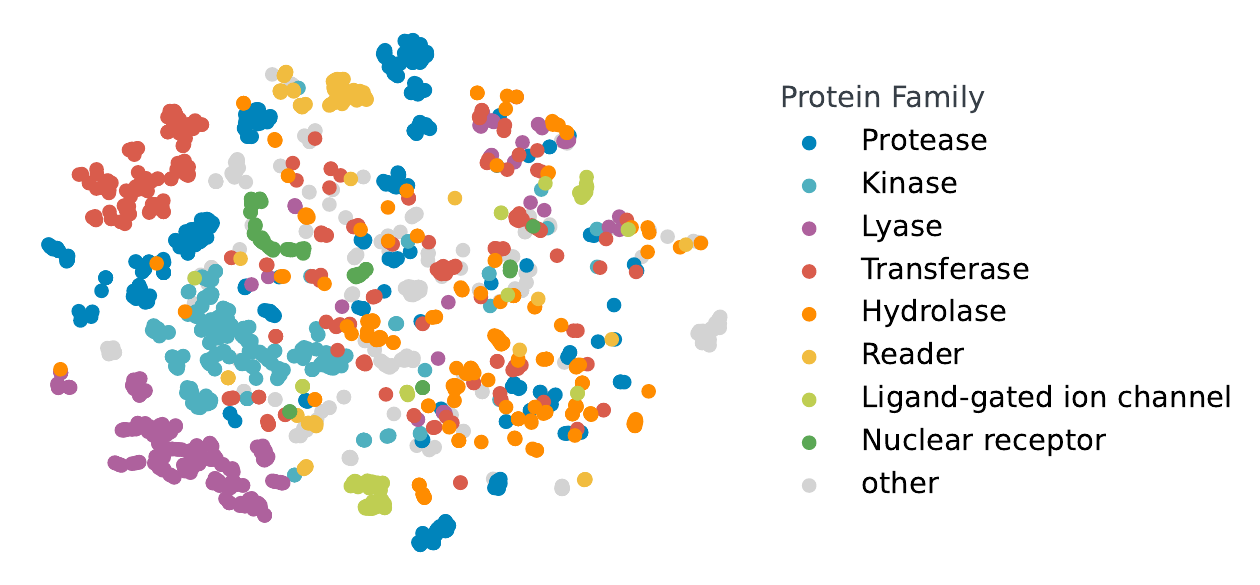}
    \caption{T-SNE embeddings of virtual node features of the best ranked pockets for each protein in the PDBbind2020 dataset colored by protein family according to ChEMBL. The eight largest protein classes are shown, remaining proteins are colored in grey.
    \label{fig:tsne_plot}}
\end{figure}

\clearpage

\section{Domain shift of the HOLO4K dataset}
\label{appsec:holo4k_artifacts}
The HOLO4K benchmark comprises a large set of protein complexes
and their annotated binding sites. HOLO4K has often been 
used as a benchmarking dataset \citep{krivak2018p2rank}, while 
it exhibits different characteristics than other
datasets, such as scPDB, COACH420 and PDBBind2020. The number
of chains per sample, i.e. PDB file, is larger than 
in these datasets (see ~\cref{fig:chain_count}), and also
the number of binding sites per entry is higher (see ~\cref{fig:bs_count}). HOLO4K contains many symmetric
units of large complexes which lead to these statistics.
Thus, for machine learning methods trained on 
scPDB, the HOLO4K dataset represents a difficult test case due
to the mentioned domain shifts. 

\vspace{1cm}

\begin{figure}[h]
    \centering
    \includegraphics[width=0.48\textwidth]{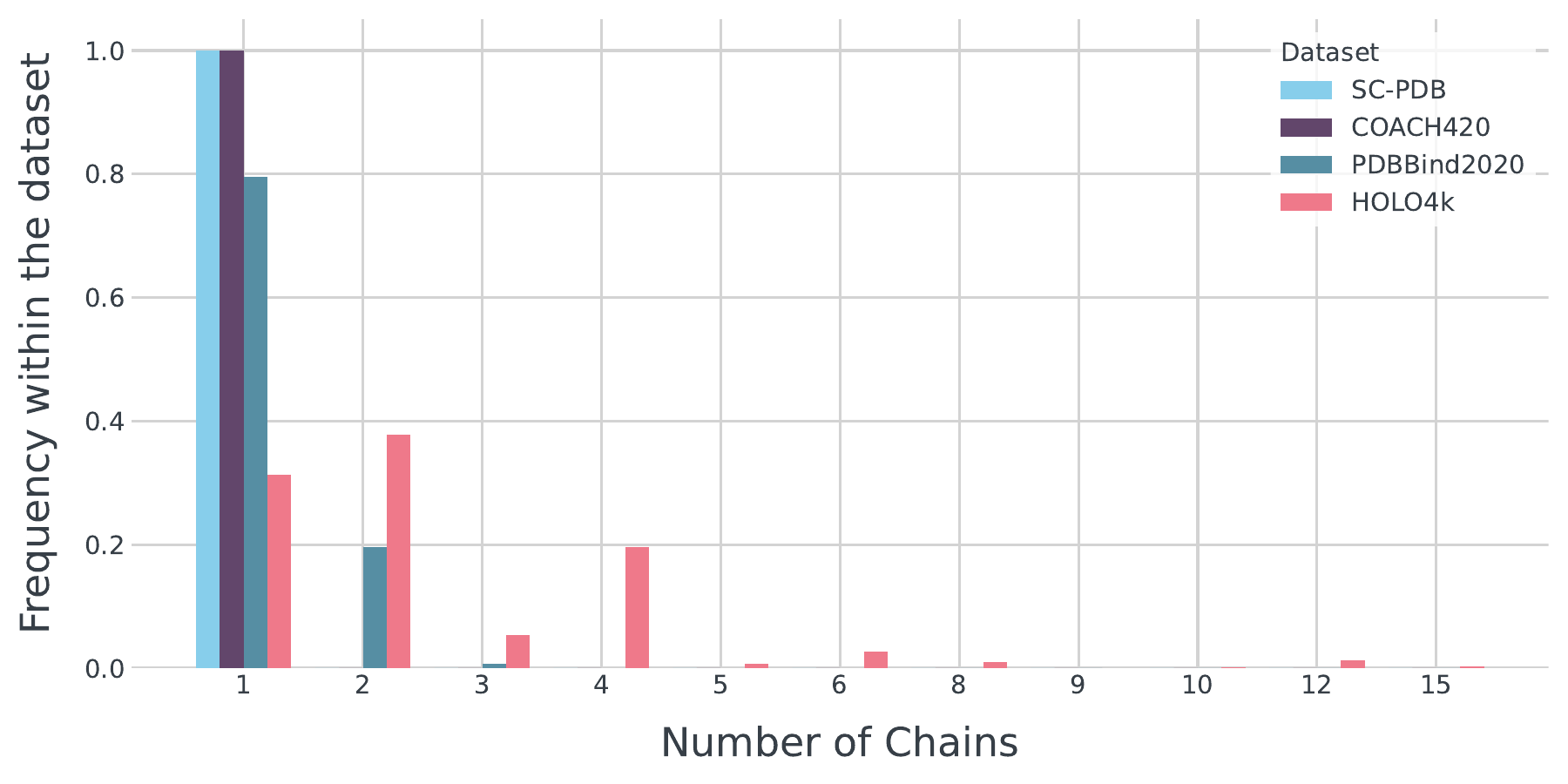}
    \includegraphics[width=0.48\textwidth]{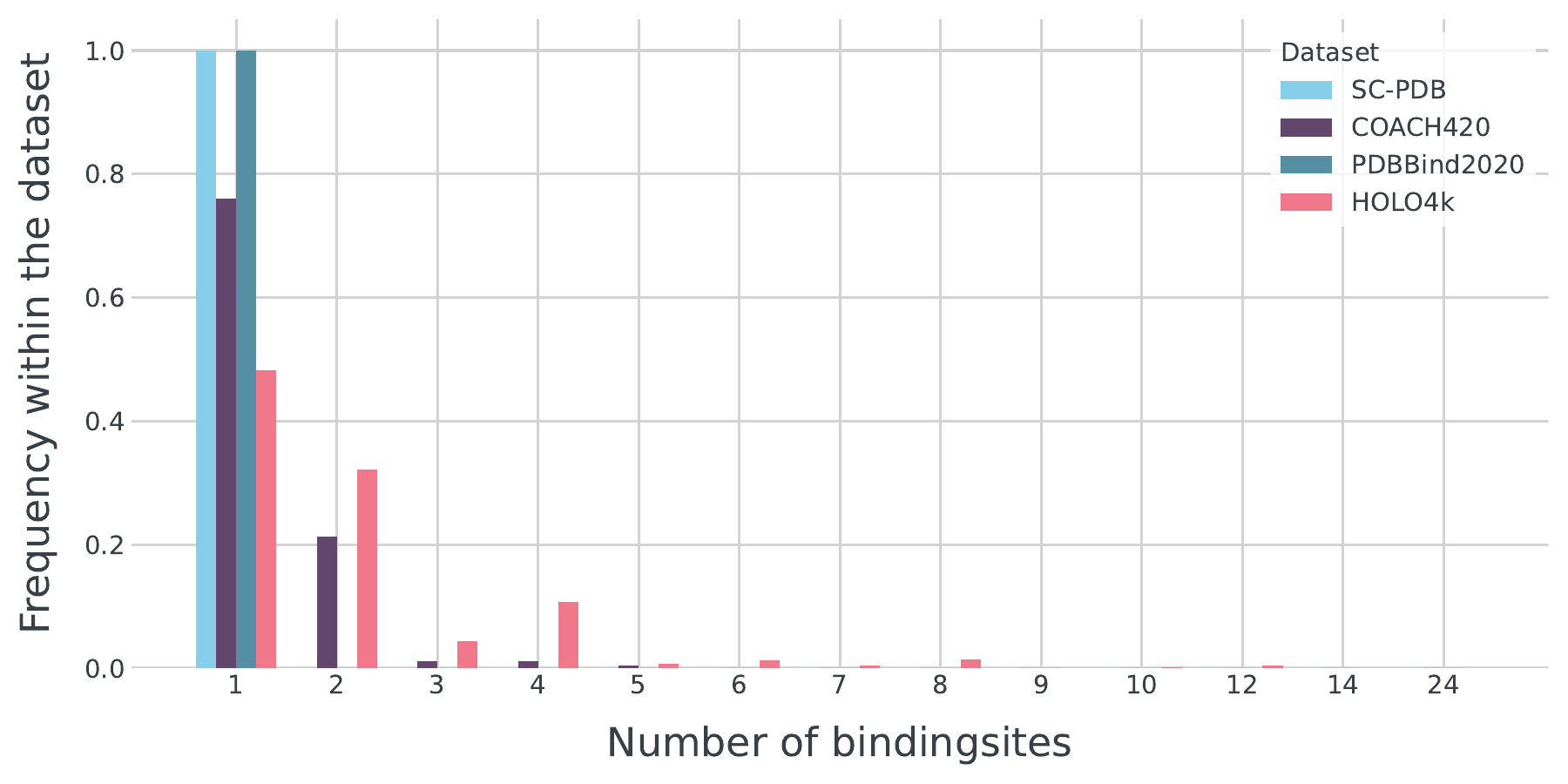}
    \caption{Histogram of the number of protein chains (left) and of the number of binding sites frequency (right) per sample for each of the datasets used 
    in this study. Note that the HOLO4K 
    dataset has highly different characteristics 
    compared to the other datasets and thus represents
    a strong domain shift for methods trained on
    scPDB. \label{fig:chain_count}  \label{fig:bs_count}}
\end{figure}

\clearpage

\section{Expressiveness of VN-EGNN}
\label{appsec:expressivity}

The expressive power of GNNs is often described in terms of their ability to distinguish non-isomorphic graphs. The Weisfeiler-Leman (WL) \citep{weisfeiler1968reduction} test, an iterative method to determine whether two attributed graphs are isomorphic, provides an upper bound to the expressiveness of GNNs. To extend the applicability of this framework to geometric graphs, \citet{Joshi2023} introduced the Geometric Weisfeiler-Leman test (GWL) which assesses whether two graphs are \textit{geometrically isomorphic}.

\textbf{Definitions} \citep{Joshi2023}: Two graphs $\mathcal{G}_1$ and $\mathcal{G}_2$ with node features $h_i^{\mathcal{G}_j}$ and coordinates $\Rx_i^{\mathcal{G}_j}$ for $j \in \{1,2\}$ are called \textit{geometrically isomorphic} if there exists an edge-preserving bijection $b: \mathcal{V} (\mathcal{G}_1) \rightarrow \mathcal{V}(\mathcal{G}_2)$ between their corresponding node indices $\mathcal{V} (\mathcal{G}_j)$, such that their geometric features are equivalent up to $E(n)$ group actions, i.e. global rotations/reflections $\BR$ and translations $\Rt$:

\begin{align}
     \left(h_{b(i)}^{\mathcal{G}_2}, \Rx_{b(i)}^{\mathcal{G}_2} \right) = \left(h_{i}^{\mathcal{G}_1}, \BR \Rx_{i}^{\mathcal{G}_1}+\Rt \right) \quad \forall i \in \mathcal{V}(\mathcal{G}_1).
\end{align}

Two graphs $\mathcal{G}_1$ and $\mathcal{G}_2$ are called \textit{$k$-hop distinct} if for all graph isomorphisms $b$, there is some node $i \in \mathcal{V}(\mathcal{G}_1), b(i) \in \mathcal{V}(\mathcal{G}_2)$ such that the corresponding $k$-hop neighborhood subgraphs $\mathcal{N}_i^{(\mathcal{G}_1,k)}$ and $\mathcal{N}_{b(i)}^{(\mathcal{G}_2,k)}$ are distinct. Otherwise, if $\mathcal{N}_i^{(\mathcal{G}_1,k)}$ and $\mathcal{N}_{b(i)}^{(\mathcal{G}_2,k)}$ are identical up to group actions for all $i \in \mathcal{V}(\mathcal{G}_1)$, we say $\mathcal{G}_1$ and $\mathcal{G}_2$ are \textit{$k$-hop identical}.

In addition to iteratively updating node colors depending on node features in the local neighborhood analogously to the WL test, GWL keeps track of $E(n)$-equivariant hash values of each node’s local geometry, i.e., distances to and angles between neighboring nodes. Thus, $k$ iterations of GWL are necessary and sufficient to distinguish any $k$-hop distinct, $(k-1)$-hop identical geometric graphs \citep{Joshi2023}. 

\begin{proposition}
    Any two geometrically distinct graphs $\mathcal{G}_1$ and $\mathcal{G}_2$, where the underlying attributed graphs are isomorphic, can be distinguished with one iteration of GWL by adding one virtual node connected to all other nodes.
\end{proposition}

\begin{proof}
For $1$-hop distinct graphs one iteration of GWL suffices to distinguish them even without virtual nodes and, thus, the proposition holds. 

Now, we assume that $\mathcal{G}_1$ and $\mathcal{G}_2$ are $k$-hop distinct and $(k-1)$-hop identical for any $k>1 \in \mathbb{N}$ and place one virtual node connected to all other nodes in an equivalent position in both graphs. 

Note that finding equivalent virtual node positions is not trivial, as there is no straight-forward way to spatially align the two graphs. However, for any $(k-1)$-hop sub-graph in $\mathcal{G}_1$ there exists an identical sub-graph in $\mathcal{G}_2$, such that the neighborhood structure between matching sub-graphs is preserved. We align the two graphs in space by overlaying one such pair of $(k-1)$-hop sub-graphs (consisting of more than two nodes that are not arranged in a straight line) and position the virtual node in the same coordinates in both aligned graphs.

Since the virtual node is connected to each node in the graph, its $1$-hop neighborhood and therefore the receptive field of the first GWL iteration contains the entire graph. Due to the $k$-hop distinctness of the graphs, there exists at least one node for which the geometric orientation relative to the matched subgraph deviates between $\mathcal{G}_1$ and $\mathcal{G}_2$. Thus, the hash values corresponding to the virtual nodes' geometric information differ, and the graphs can be distinguished by only one iteration of GWL.
\end{proof}

As $k$ iterations of GWL act as an upper bound on the expressiveness of a $k$-layer geometric GNN, we propose that one layer of VN-EGNN is sufficient to distinguish two $k$-hop distinct graphs while without virtual nodes $k$ EGNN layers are necessary to complete the same task.

We demonstrate this on the example of $n$-chain geometric graphs, where each pair of graphs comprises $n$ nodes arranged in a line and two end points with distinct orientations (\Cref{fig:k_chains}). These graphs are ($\lfloor{\frac{n}{2}} \rfloor + 1$)-hop distinct and should therefore be distinguishable by ($\lfloor{\frac{n}{2}} \rfloor + 1$) EGNN layers or ($\lfloor{\frac{n}{2}} \rfloor + 1$) iterations of GWL.

\begin{figure}[ht]
    \centering
    \includegraphics[width=0.99\textwidth]{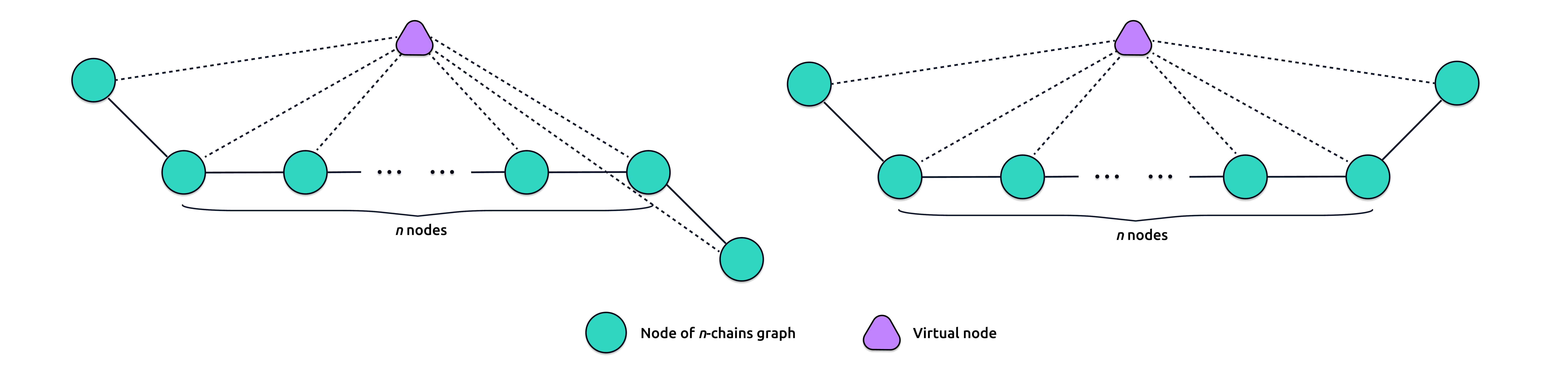}
    \caption{A pair of $n$-chain geometric graphs consisting of $n$ nodes arranged in a line and two end points with opposite orientations. Without the addition of a virtual node, these graphs are ($\lfloor{\frac{n}{2}} \rfloor + 1$)-hop distinct.}
    \label{fig:k_chains}
\end{figure}

We trained EGNNs with an increasing number of layers to classify $4$-chain graphs with and without the addition of a virtual node. The virtual node was placed in an equivalent position for both graphs, which was selected randomly on a sphere centered at the midpoint of the $n$-chain and surrounding each graph. Note that for the experiments including virtual nodes, we did not use the heterogeneous message passing scheme as described in \Cref{subsec:vn_message_passing}, but apply the EGNN to the entire graph, including virtual nodes, at once.

The results shown in \Cref{tab:k_chains} demonstrate that, as expected, $3$ layers of EGNN are necessary to distinguish the $4$-chain graphs while after adding a virtual node, one iteration is sufficient for correct classification, indicating the increased expressiveness of VN-EGNN.
Although we used the setting of \citet{Joshi2023}, we could not 
reproduce their finding that $6$ EGNN layers are necessary to solve this
task, which they explained with possible oversmoothing or oversquashing effects. The differences might arise from the use of different features dimensions, which is why we include results for $5$ different feature dimensions.

\begin{table}[ht]
    \centering
    \resizebox{0.99\textwidth}{!}{
    \begin{tabular}{lccccccccc}
        \toprule 
   &  Dim. & 1 Layer & 2 Layers  & 3 Layers & 4 Layers & 5 Layers  & 6 Layers & 7 Layers  & 8 Layers\\ \midrule
  & 8& 50.0 $\pm$ 0.0& 50.0 $\pm$ 0.0& 50.0 $\pm$ 0.0& 98.0 $\pm$ 9.8& 94.0 $\pm$ 16.2& 93.0 $\pm$ 17.3& 99.5 $\pm$ 5.0& 99.5 $\pm$ 5.0\\ 
  & 16& 50.0 $\pm$ 0.0& 50.0 $\pm$ 0.0& 86.0 $\pm$ 22.4& 97.5 $\pm$ 10.9& 99.5 $\pm$ 5.0& 99.5 $\pm$ 5.0& 99.5 $\pm$ 5.0& 100.0 $\pm$ 0.0\\ 
 \bf EGNN & 32& 50.0 $\pm$ 0.0& 50.0 $\pm$ 0.0& 56.5 $\pm$ 16.8& 50.0 $\pm$ 0.0& 50.0 $\pm$ 0.0& 96.5 $\pm$ 12.8& 99.0 $\pm$ 7.0& 93.5 $\pm$ 16.8\\ 
  & 64& 50.0 $\pm$ 0.0& 50.0 $\pm$ 0.0& 100.0 $\pm$ 0.0& 99.0 $\pm$ 7.0& 100.0 $\pm$ 0.0& 99.0 $\pm$ 7.0& 100.0 $\pm$ 0.0& 100.0 $\pm$ 0.0\\ 
  & 128& 50.0 $\pm$ 0.0& 50.0 $\pm$ 0.0& 96.5 $\pm$ 12.8& 98.5 $\pm$ 8.5& 95.0 $\pm$ 15.0& 99.5 $\pm$ 5.0& 99.5 $\pm$ 5.0& 99.5 $\pm$ 5.0\\ 
        \midrule
  & 8& 65.5 $\pm$ 23.1& 50.0 $\pm$ 0.0& 84.5 $\pm$ 23.1& 92.5 $\pm$ 17.9& 64.0 $\pm$ 22.4& 97.0 $\pm$ 11.9& 86.5 $\pm$ 23.3& 97.5 $\pm$ 10.9\\ 
  & 16& 86.0 $\pm$ 23.5& 95.0 $\pm$ 15.0& 98.5 $\pm$ 8.5& 99.5 $\pm$ 5.0& 99.5 $\pm$ 5.0& 98.0 $\pm$ 9.8& 99.5 $\pm$ 5.0& 100.0 $\pm$ 0.0\\ 
 \bf VN-EGNN & 32& 95.0 $\pm$ 15.0& 100.0 $\pm$ 0.0& 99.5 $\pm$ 5.0& 99.5 $\pm$ 5.0& 100.0 $\pm$ 0.0& 100.0 $\pm$ 0.0& 100.0 $\pm$ 0.0& 100.0 $\pm$ 0.0\\ 
  & 64& 97.5 $\pm$ 10.9& 100.0 $\pm$ 0.0& 99.5 $\pm$ 5.0& 99.5 $\pm$ 5.0& 99.0 $\pm$ 7.0& 100.0 $\pm$ 0.0& 100.0 $\pm$ 0.0& 99.5 $\pm$ 5.0\\ 
  & 128& 99.0 $\pm$ 7.0& 99.5 $\pm$ 5.0& 99.5 $\pm$ 5.0& 99.0 $\pm$ 7.0& 99.5 $\pm$ 5.0& 99.5 $\pm$ 5.0& 99.0 $\pm$ 7.0& 99.0 $\pm$ 7.0 \\
        \bottomrule
    \end{tabular}}
    \caption{Classification accuracy of EGNNs with and without virtual nodes and increasing node embedding dimensions on $4$-chain geometric graphs. The standard deviation across 100 training re-runs is indicated with $\pm$ and column "Dim" indicates the used node feature dimension. Note that VN-EGNN can distinguish these graphs already with one message passing layer (see columns "1 Layer" and "2 Layers").}
    \label{tab:k_chains}
\end{table}

\end{document}